\documentclass[11pt]{article}
\pdfoutput=1
\usepackage{comment, fullpage}
\usepackage{amssymb}
\usepackage{graphicx}
\usepackage{url}
\usepackage{setspace}
\usepackage[pdftex,bookmarksnumbered,bookmarksopen,
colorlinks,citecolor=blue,linkcolor=blue,urlcolor=blue]{hyperref}
\usepackage{framed}
\usepackage{xcolor}
\usepackage{soul}

\usepackage[utf8]{inputenc} 
\usepackage[T1]{fontenc}    
\usepackage{hyperref}       
\usepackage{url}            
\usepackage{booktabs}       
\usepackage{amsfonts}       
\usepackage{nicefrac}       
\usepackage{microtype}      
\usepackage{amsmath}

\usepackage{subcaption}
\usepackage{graphicx}
\usepackage{appendix}
\usepackage{amsmath,amsfonts,amsthm}
\usepackage{algorithm,algorithmic}
\usepackage[algo2e,ruled,vlined]{algorithm2e}
\usepackage{mdwlist}
\usepackage{xspace}
\usepackage{color}
\usepackage{mathrsfs}
\usepackage{amssymb}
\usepackage{wrapfig}
\usepackage{bm}
\usepackage{amsmath}
\usepackage{booktabs}
\usepackage{comment}
\newcommand{\Appendix}[1]{the full version for}

\newtheorem{theorem}{Theorem}
\newtheorem{lemma}[theorem]{Lemma}

\newtheorem{condition}{Condition}
\newtheorem{assumption}{Assumption}

\renewcommand{\v}{\mathbf{v}}
\newcommand{\w}{\mathbf{w}}
\newcommand{\x}{\mathbf{x}}
\newcommand{\y}{\mathbf{y}}

\newcommand{\A}{\mathbf{A}}
\newcommand{\B}{\mathbf{B}}
\newcommand{\C}{\mathbf{C}}

\newcommand{\E}{\mathbf{E}}

\newcommand{\I}{\mathbf{I}}

\newcommand{\M}{\mathbf{M}}

\newcommand{\R}{\mathbb{R}}

\newcommand{\T}{\mathbf{T}}
\newcommand{\U}{\mathbf{U}}
\newcommand{\V}{\mathbf{V}}
\newcommand{\W}{\mathbf{W}}
\newcommand{\X}{\mathbf{X}}
\newcommand{\Y}{\mathbf{Y}}
\newcommand{\Z}{\mathbf{Z}}
\newcommand{\rank}{\textup{\textsf{rank}}}

\newcommand{\bLambda}{\mathbf{\Lambda}}

\newcommand{\0}{\mathbf{0}}

\renewcommand{\comment}[1]{}

\newcommand{\cA}{\mathcal{A}}

\newcommand{\cI}{\mathcal{I}}

\newcommand{\cP}{\mathcal{P}}

\newcommand{\cS}{\mathcal{S}}
\newcommand{\cT}{\mathcal{T}}
\newcommand{\cU}{\mathcal{U}}
\newcommand{\cV}{\mathcal{V}}
\newcommand{\cW}{\mathcal{W}}
\newcommand{\cY}{\mathcal{Y}}
\newcommand{\tY}{\widetilde{\Y}}

\newcommand{\svd}{\textup{\textsf{svd}}}
\newcommand{\bSigma}{\bm{\Sigma}}
\newcommand{\diag}{\textup{\textsf{diag}}}
\newcommand{\bbE}{\mathbb{E}}

\DeclareMathOperator*{\argmax}{argmax}
\DeclareMathOperator*{\argmin}{argmin}

\title{Deep Neural Networks with Multi-Branch Architectures Are\\ Less Non-Convex}
\usepackage{times}

\author{
Hongyang Zhang \\ Carnegie Mellon University \\ \small{hongyanz@cs.cmu.edu} \and
Junru Shao  \\ Carnegie Mellon University \\ \small{junrus@cs.cmu.edu} \and
Ruslan Salakhutdinov  \\ Carnegie Mellon University \\ \small{rsalakhu@cs.cmu.edu}
}

\date{}

\begin{document}
\maketitle

\begin{abstract}
Several recently proposed architectures of neural networks such as ResNeXt, Inception, Xception, SqueezeNet and Wide ResNet are based on the designing idea of having multiple branches and have demonstrated improved performance in many applications. We show that one cause for such success is due to the fact that the multi-branch architecture is less non-convex in terms of duality gap. The duality gap measures the degree of intrinsic non-convexity of an optimization problem: smaller gap in relative value implies lower degree of intrinsic non-convexity. The challenge is to quantitatively measure the duality gap of highly non-convex problems such as deep neural networks. In this work, we provide strong guarantees of this quantity for two classes of network architectures. For the neural networks with \emph{arbitrary activation functions}, multi-branch architecture and a variant of hinge loss, we show that the duality gap of both population and empirical risks shrinks to zero as the number of branches increases. This result sheds light on better understanding the power of over-parametrization where increasing the network width tends to make the loss surface less non-convex. For the neural networks with linear activation function and $\ell_2$ loss, we show that the duality gap of empirical risk is zero. Our two results work for \emph{arbitrary depths} and \emph{adversarial data}, while the analytical techniques might be of independent interest to non-convex optimization more broadly. Experiments on both synthetic and real-world datasets validate  our results.
\end{abstract}

\section{Introduction}
Deep neural networks are a central object of study in machine learning, computer vision, and many other domains.
They have substantially improved over conventional learning algorithms in many areas, including speech recognition, object detection, and natural language processing~\cite{Goodfellow-et-al-2016}. The focus of this work is to investigate the duality gap of deep neural networks. The duality gap is the discrepancy between the optimal values of primal and dual problems. While it has been well understood for convex optimization, little is known for non-convex problems. 
A smaller duality gap in relative value typically implies that the problem itself is less non-convex, and thus is easier to optimize.\footnote{Although zero duality gap can be attained for some non-convex optimization problems~\cite{Balcan2017optimal,overton1992sum,beck2006strong}, they are in essence convex problems by considering the dual and bi-dual problems, which are always convex. So these problems are relatively easy to optimize compared with other non-convex ones.} Our results establish that:
\emph{Deep neural networks with multi-branch architecture have small duality gap in relative value.}

Our study is motivated by the computational difficulties of deep neural networks due to its non-convex nature. While many works have witnessed the power of local search algorithms for deep neural networks~\cite{brutzkus2017globally}, these algorithms typically converge to a suboptimal solution in the worst cases according to various empirical observations~\cite{Shalev-Shwartz2017failures,Goodfellow-et-al-2016}. It is reported that for a single-hidden-layer neural network, when the number of hidden units is small, stochastic gradient descent may get easily stuck at the poor local minima~\cite{ge2017learning,safran2017spurious}. Furthermore, there is significant evidence indicating that when the networks are deep enough, bad saddle points do exist~\cite{anandkumar2016efficient} and might be hard to escape~\cite{blum1989training,dasgupta1995complexity,bartlett1999hardness,anandkumar2016efficient}.

Given the computational obstacles, several efforts have been devoted to designing new architectures to alleviate the above issues, including
over-parametrization~\cite{brutzkus2018sgd,soltanolkotabi2017theoretical,du2018power,li2017algorithmic,arora2018optimization,Neyshabur2018towards} and multi-branch architectures~\cite{szegedy2017inception,chollet2016xception,xie2017aggregated,iandola2016squeezenet,veit2016residual}. Empirically, increasing the number of hidden units of a single-hidden-layer network encourages the first-order methods to converge to a global solution, which probably supports the folklore that the loss surface of a wider network looks more ``convex'' (see Figure \ref{figure: experiment A}). Furthermore, several recently proposed architectures, including ResNeXt~\cite{xie2017aggregated}, Inception~\cite{szegedy2017inception}, Xception~\cite{chollet2016xception}, SqueezeNet~\cite{iandola2016squeezenet} and Wide ResNet~\cite{zagoruyko2016wide} are based on having multiple branches and have demonstrated substantial improvement over many of the existing models in many applications. In this work, we show that one cause for such success is due to the fact that the loss of multi-branch network is less non-convex in terms of duality gap.  

\begin{figure}
  \begin{subfigure}{0.24\textwidth}
    \includegraphics[width=\textwidth]{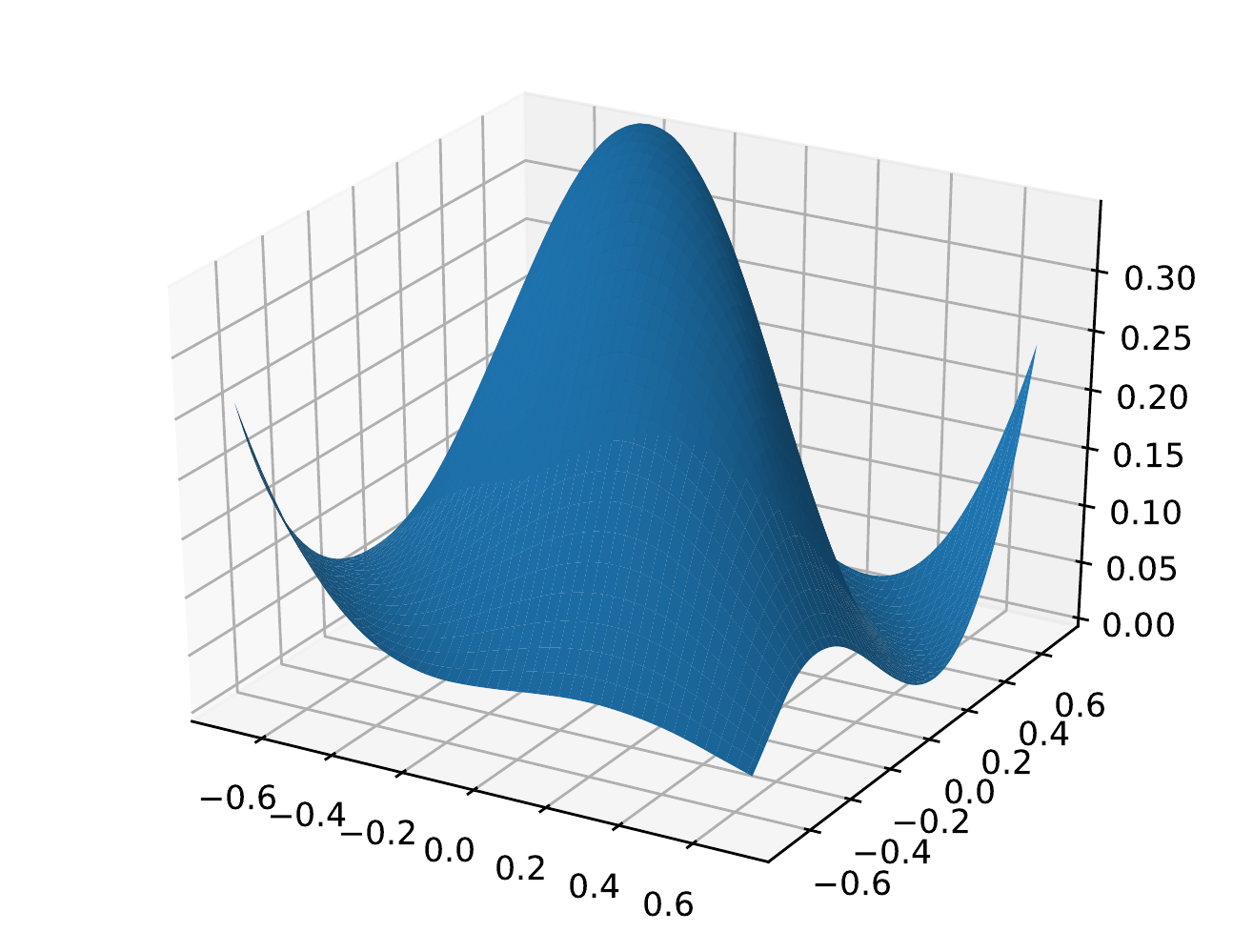}
    \caption{$I=10$.}
  \end{subfigure}
  \begin{subfigure}{0.24\textwidth}
    \includegraphics[width=\textwidth]{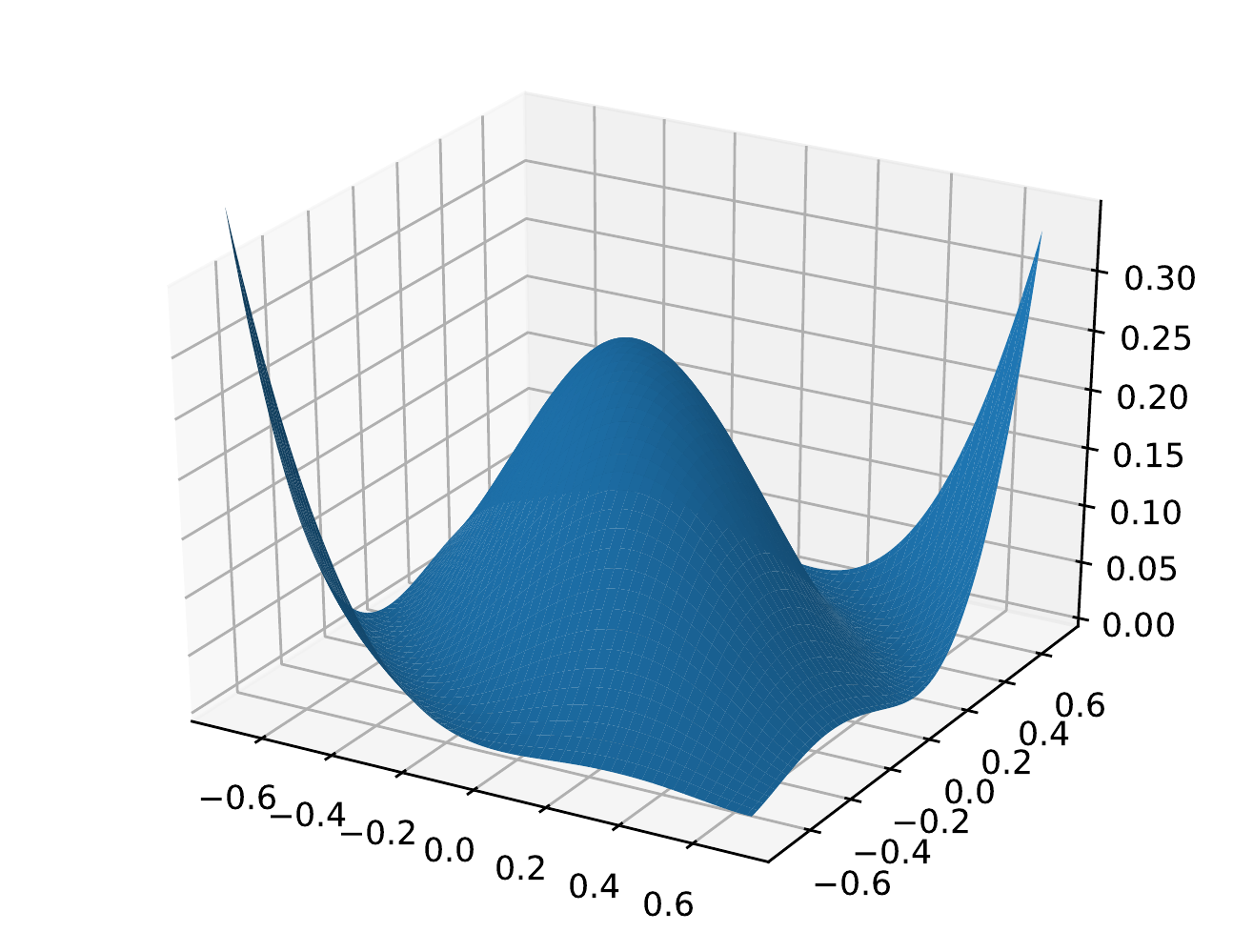}
    \caption{$I=30$.}
  \end{subfigure}
  \begin{subfigure}{0.24\textwidth}
    \includegraphics[width=\textwidth]{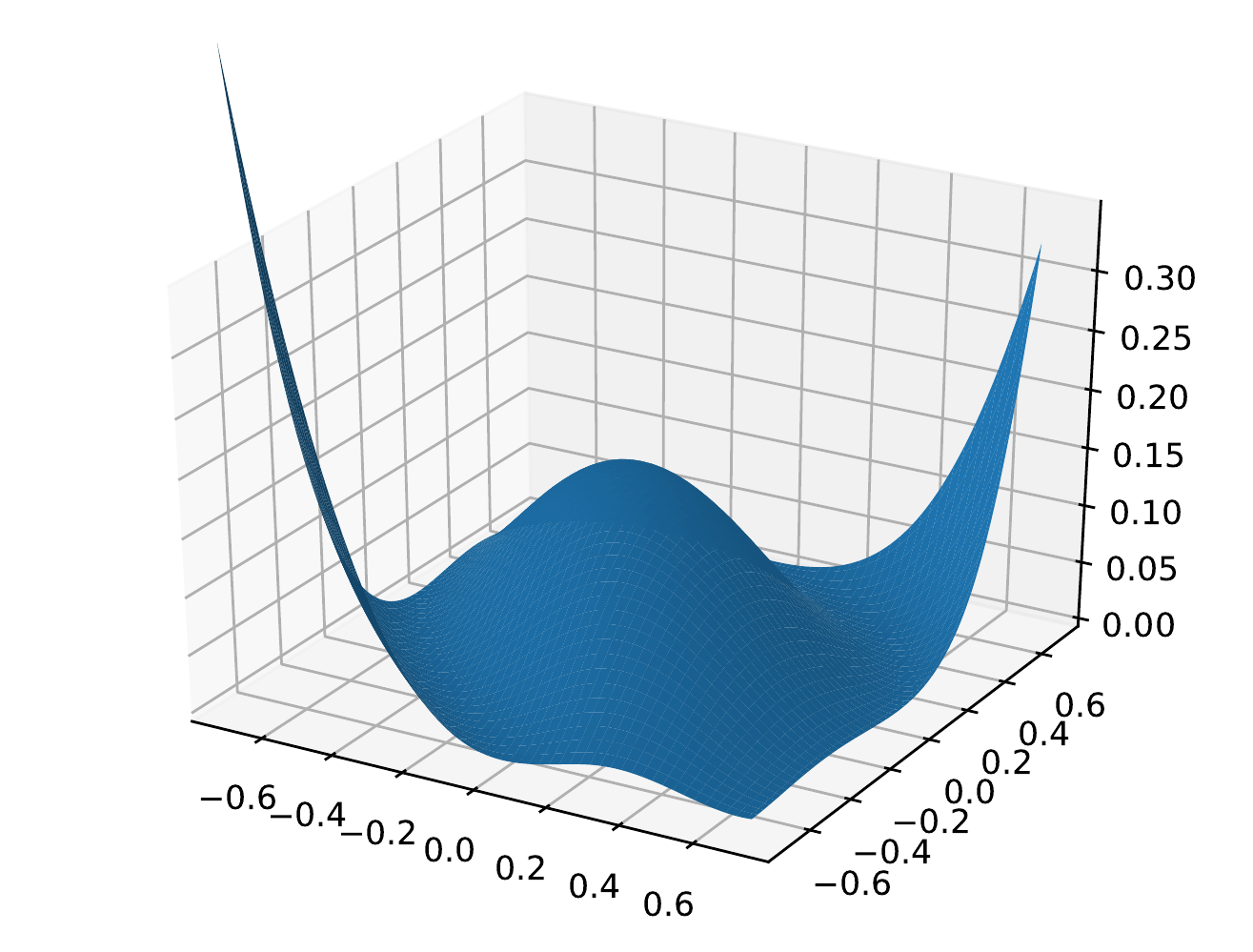}
    \caption{$I=70$.}
  \end{subfigure}
  \begin{subfigure}{0.24\textwidth}
    \includegraphics[width=\textwidth]{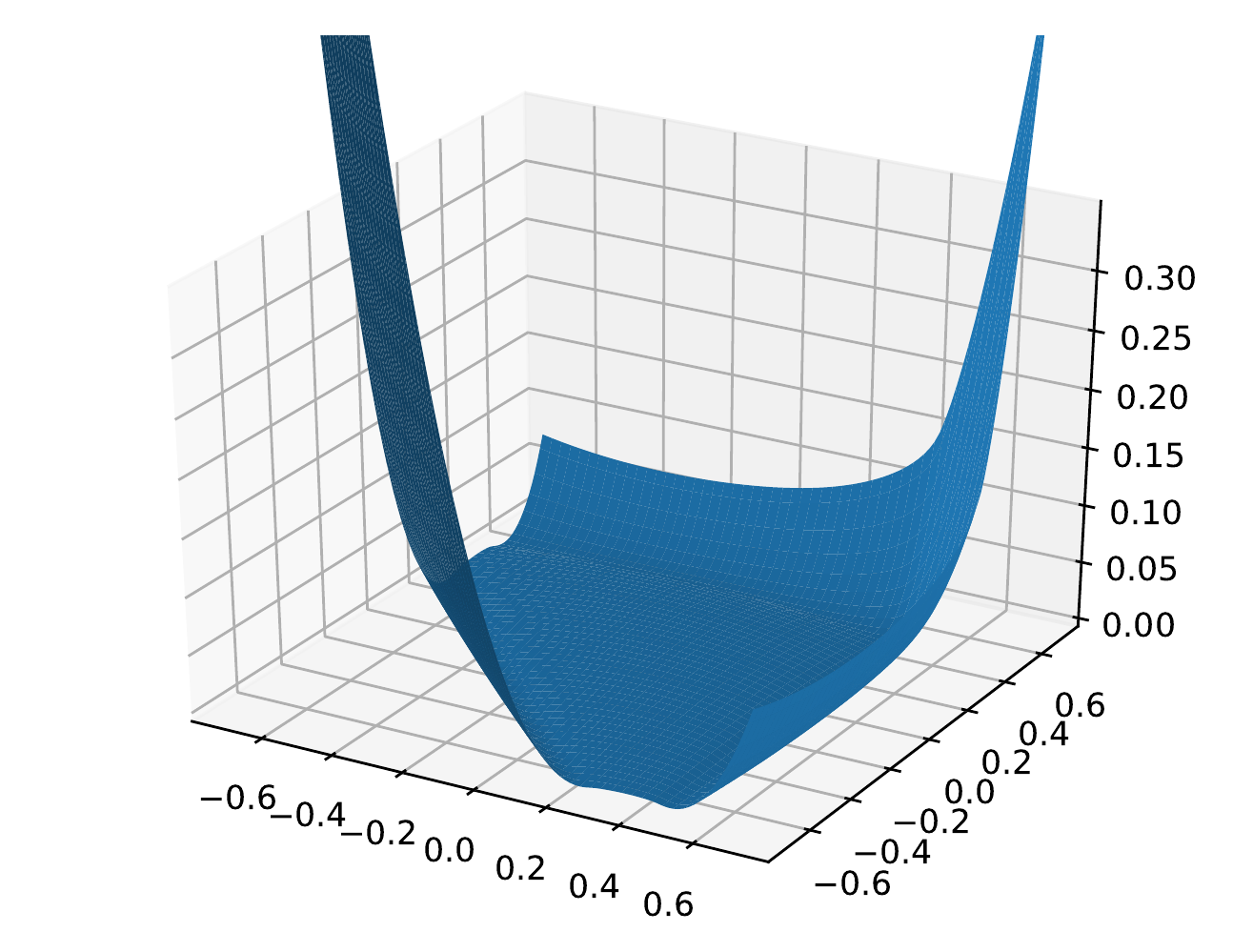}
    \caption{$I=1,000$.}
  \end{subfigure}
  \caption{\small The loss surface of one-hidden-layer ReLU network projected onto a 2-d plane, which is spanned by three points to which the SGD algorithm converges according to three different initialization seeds. It shows that as the number of hidden neurons $I$ increases, the landscape becomes less non-convex.}
  \label{figure: experiment A}
\end{figure}

\medskip
\noindent{\textbf{Our Contributions.}} This paper provides both theoretical and experimental results for the population and empirical risks of deep neural networks by estimating the duality gap.

First, we study the duality gap of deep neural networks with \emph{arbitrary} activation functions, \emph{adversarial} data distribution, and multi-branch architecture (see Theorem \ref{theorem: deep non-linear neural networks}). The multi-branch architecture is general, which includes the classic one-hidden-layer architecture as a special case (see Figure \ref{figure: architecture}). By Shapley-Folkman lemma, we show that the duality gap of both population and empirical risks shrinks to zero as the number of branches increases. Our result provides better understanding of various state-of-the-art architectures such as ResNeXt, Inception, Xception, SqueezeNet, and Wide ResNet.

Second, we prove that the strong duality (a.k.a. zero duality gap) holds for the empirical risk of deep linear neural networks (see Theorem \ref{theorem: local=global}). To this end, we develop multiple new proof techniques, including \emph{reduction to low-rank approximation} and \emph{construction of dual certificate} (see Section \ref{section: Our Techniques and Proof Sketches}).

Finally, we empirically study the loss surface of multi-branch neural networks. Our experiments verify our theoretical findings.


\comment{
\begin{figure}
\centering
\includegraphics[width=1\textwidth]{bridge.pdf}
\caption{\small The duality gap of deep neural network implied in this paper, where the blue curve is the landscape of deep neural networks while the red dashed curve represents the landscape of dual problems. \textbf{(a).} Strong duality holds for deep linear neural networks (see Theorem \ref{theorem: local=global}). \textbf{(b).} The duality gap can be bounded for deep non-linear neural networks with parallel-path architectures (see Theorem \ref{theorem: deep non-linear neural networks}).}
\label{figure: strong duality of non-convex problems}
\end{figure}
}

\medskip
\noindent{\textbf{Notation.}} We will use bold capital letter to represent matrix and lower-case letter to represent scalar. Specifically, let $\I$ be the identity matrix and denote by $\0$ the all-zero matrix. Let $\{\W_i\in\R^{d_i\times d_{i-1}}:i=1,2,...,H\}$ be a set of network parameters, each of which represents the connection weights between the $i$-th and $(i+1)$-th layers of neural network. We use $\W_{:,t}\in\R^{n_1\times 1}$ to indicate the $t$-th column of $\W$. We will use $\sigma_i(\W)$ to represent the $i$-th largest singular value of matrix $\W$. Given skinny SVD $\U\bSigma\V^T$ of matrix $\W$, we denote by $\svd_r(\W)=\U_{:,1:r}\bSigma_{1:r,1:r}\V_{:,1:r}^T$ the truncated SVD of $\W$ to the first $r$ singular values. For matrix norms, denote by $\|\W\|_{\cS_H}=\left(\sum_i\sigma_i^H(\W)\right)^{1/H}$ the matrix Schatten-$H$ norm. Nuclear norm and Frobenius norm are special cases of Schatten-$H$ norm: $\|\W\|_*=\|\W\|_{\cS_1}$ and $\|\W\|_F=\|\W\|_{\cS_2}$. We use $\|\W\|$ to represent the matrix operator norm, i.e., $\|\W\|=\sigma_1(\W)$, and denote by $\rank(\W)$ the rank of matrix $\W$. Denote by $\mathsf{Row}(\W)$ the span of rows of $\W$. Let $\W^\dag$ be the Moore-Penrose pseudo-inverse of $\W$.

 For convex matrix function $K(\cdot)$, we denote by $K^*(\bLambda)=\max_\M\langle\bLambda,\M\rangle-K(\M)$ the conjugate function of $K(\cdot)$ and $\partial K(\cdot)$ the sub-differential. We use $\diag(\sigma_1,...,\sigma_r)$ to represent a $r\times r$ diagonal matrix with diagonal entries $\sigma_1,...,\sigma_r$. Let $d_{\min}=\min\{d_i:i=1,2,...,H-1\}$, and $[I]=\{1,2,...,I\}$. For any two matrices $\A$ and $\B$ of matching dimensions, we denote by $[\A,\B]$ the concatenation of $\A$ and $\B$ along the row and $[\A;\B]$ the concatenation of two matrices along the column.

\section{Duality Gap of Multi-Branch Neural Networks}

\label{section: Approximate Strong Duality of Deep Non-Linear Networks}

We first study the duality gap of neural networks in a classification setting. We show that the wider the network is, the smaller the duality gap becomes.

\medskip
\noindent{\textbf{Network Setup.}}
The output of our network follows from a multi-branch architecture (see Figure \ref{figure: architecture}):
\begin{equation*}
f(\w;\x)=\frac{1}{I}\sum_{i=1}^I f_i(\w_{(i)};\x),\quad \w_{(i)}\in\cW_i, \quad (\cW_i \text{ is convex set})
\end{equation*}
where $\w$ is the concatenation of all network parameters $\{\w_{(i)}\}_{i=1}^I$, $\x\in\R^{d_0}$ is the input instance, $\{\cW_i\}_{i=1}^I$ is the parameter space, and $f_i(\w_{(i)};\cdot)$ represents an $\R^{d_0}\rightarrow\R$ continuous mapping by a sub-network which is allowed to have \emph{arbitrary} architecture such as convolutional and recurrent neural networks. As an example, $f_i(\w_{(i)};\cdot)$ can be in the form of a $H_i$-layer feed-forward sub-network:
\begin{equation*}
f_i(\w_{(i)};\x)=\w_i^\top\psi_{H_i}(\W_{H_i}^{(i)}...\psi_1(\W_1^{(i)}\x))\in\R,\ \w_{(i)}=[\w_i;\mathsf{vec}(\W_1^{(i)});...;\mathsf{vec}(\W_{H_i}^{(i)})]\in\R^{p_i}.
\end{equation*}
Hereby, the functions $\psi_k(\cdot),k=1,2,...,H_i$ are allowed to encode \emph{arbitrary} form of continuous element-wise non-linearity (and linearity) after each matrix multiplication, such as sigmoid, rectification, convolution, while the number of layers $H_i$ in each sub-network can be \emph{arbitrary} as well. When $H_i=1$ and $d_{H_i}=1$, i.e., each sub-network in Figure \ref{figure: architecture} represents one hidden unit, the architecture $f(\w;\x)$ reduces to a  one-hidden-layer network. We apply the so-called \emph{$\tau$-hinge loss}~\cite{awasthi2016learning,balcan2017sample} on the top of network output for label $y\in\{-1,+1\}$:
\begin{equation}
\label{equ: hinge loss}
\ell_\tau(\w;\x,y):=\max\left(0,1-\frac{y\cdot f(\w;\x)}{\tau}\right),\quad \tau>0.
\end{equation}
The $\tau$-hinge loss has been widely applied in active learning of classifiers and margin based learning~\cite{awasthi2016learning,balcan2017sample}.
When $\tau=1$, it reduces to the classic hinge loss~\cite{liang2018understanding,brutzkus2018sgd,laurent2017multilinear}.

We make the following assumption on the margin parameter $\tau$, which states that the parameter $\tau$ is sufficiently large.
\begin{assumption}[Parameter $\tau$]
\label{assumption: bounded margin}
For sample $(\x,y)$ drawn from distribution $\cP$, we have $\tau>y\cdot f(\w;\x)$ for all $\w\in\cW_1\times \cW_2\times...\times \cW_I$ with probability measure $1$.
\end{assumption}

\begin{wrapfigure}{R}{7cm}
\vspace{-0.3cm}
\includegraphics[width=7cm]{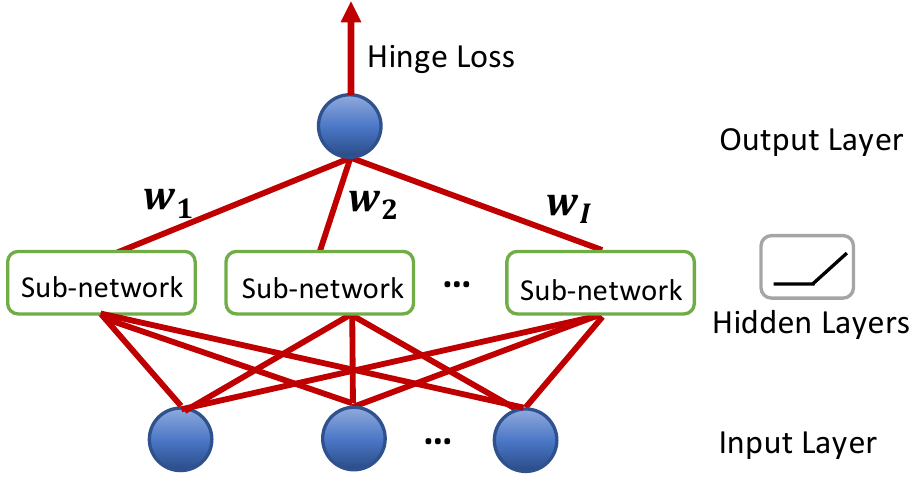}
\caption{\small Multi-branch architecture, where the sub-networks are allowed to have arbitrary architectures, depths, and continuous activation functions. In the extreme case when the sub-network is chosen to have a single  neuron, the multi-branch architecture reduces to a single-hidden-layer neural network.}
\label{figure: architecture}
\vspace{-0.5cm}
\end{wrapfigure}

We further empirically observe that using smaller values of the parameter $\tau$ and
other loss functions support our theoretical result as well (see experiments in Section \ref{section: experiments}). It is an interesting open question to extend our theory to more general losses in the future.

To study how close these generic neural network architectures approach the family of convex functions, we analyze the duality gap of minimizing the risk w.r.t. the loss \eqref{equ: hinge loss} with an extra regularization constraint. The normalized duality gap is a measure of intrinsic non-convexity of a given function~\cite{bertsekas1982estimates}: the gap is zero when the given function itself is convex, and is large when the loss surface is far from the convexity intrinsically. Typically, the closer the network approaches to the family of convex functions, the easier we can optimize the network.

\comment{
or equivalently, the \emph{non-linear margin} $y\cdot f(\w;\x)$ is small (in analogy with the \emph{linear margin} $y\cdot \langle\w,\x\rangle$ for the linear classification problems). That is, the two classes of mapped instances through the sub-networks \emph{cannot} be separated by a large margin. The assumption naturally captures the worst case when the two classes of mapped instances are tangled heavily.}

\medskip
\noindent{\textbf{Multi-Branch Architecture.}} Our analysis of multi-branch neural networks is built upon tools from non-convex geometric analysis --- Shapley–Folkman lemma. Basically, the Shapley–Folkman lemma states that the sum of constrained non-convex functions is close to being convex. A neural network is an ideal target to apply this lemma to: the width of network is associated with the number of summand functions. So intuitively, the wider the neural network is, the smaller the duality gap will be. In particular, we study the following non-convex problem concerning the population risk:
\begin{equation}
\label{equ: problem (P)}
\min_{\w\in\cW_1\times...\times\cW_I} \bbE_{(\x,y)\sim \cP}[\ell_\tau(\w;\x,y)],\ \ \mbox{s.t.}\ \ \frac{1}{I}\sum_{i=1}^I h_i(\w_{(i)})\le K,
\end{equation}
where $h_i(\cdot),i\in[I]$ are convex regularization functions, e.g., the weight decay, and $K$ can be arbitrary such that the problem is feasible.
Correspondingly, the dual problem of problem \eqref{equ: problem (P)} is a one-dimensional convex optimization problem:\footnote{Although problem \eqref{equ: problem (D)} is convex, it does not necessarily mean the problem can be solved easily. This is because computing $\mathcal{Q}(\lambda)$ is a hard problem. So rather than trying to solve the convex dual problem, our goal is to study the duality gap in order to understand the degree of non-convexity of the problem.}
\begin{equation}
\label{equ: problem (D)}
\begin{split}
\max_{\lambda\ge 0} \mathcal{Q}(\lambda)-\lambda K,\quad \textup{where } \mathcal{Q}(\lambda):=\inf_{\w\in\cW_1\times ...\times \cW_I}\bbE_{(\x,y)\sim \cP} [\ell_\tau(\w;\x,y)]+\frac{\lambda}{I} \sum_{i=1}^I h_i(\w_{(i)}).
\end{split}
\end{equation}

For $\widetilde\w\in\cW_i$, denote by
\begin{equation*}
\widetilde{f}_i(\widetilde\w):=\hspace{-0.2cm}\inf_{a^j,\w_{(i)}^j\in\cW_i}\left\{\sum_{j=1}^{p_i+1}a^j \bbE_{(\x,y)\sim \cP} \left(1-\frac{y\cdot f_i(\w_{(i)}^j;\x)}{\tau}\right)\hspace{-0.1cm}:\widetilde{\w}=\hspace{-0.2cm}\sum_{j=1}^{p_i+1}a^j\w_{(i)}^j,\sum_{j=1}^{p_i+1}a^j=1,a^j\ge 0\right\}
\end{equation*}
the convex relaxation of function $\bbE_{(\x,y)\sim \cP} [1-y\cdot f_i(\cdot;\x)/\tau]$ on $\cW_j$. For $\widetilde{\w}\in\cW_i$, we also define
\begin{equation*}
\widehat{f}_i(\widetilde{\w}):=\inf_{\w_{(i)}\in\cW_i}\left\{\bbE_{(\x,y)\sim \cP} \left(1-\frac{y\cdot f_i(\w_{(i)};\x)}{\tau}\right):h_i(\w_{(i)})\le h_i(\widetilde{\w})\right\}.
\end{equation*}

\comment{
\begin{figure}
\centering
\includegraphics[scale=0.68]{architecture.pdf}
\caption{\small Comparison of network architectures. \textbf{(a).} ResNet~\cite{he2016deep}. \textbf{(b).} ResNeXt~\cite{xie2017aggregated}. \textbf{(c).} Our network. From left to right, the network architecture encourages higher level of paralleling.}
\label{figure: architectures}
\end{figure}
}

Our main results for multi-branch neural networks are as follows:
\begin{theorem}
\label{theorem: deep non-linear neural networks}
Denote by $\inf(\mathbf{P})$ the minimum of primal problem \eqref{equ: problem (P)} and $\sup(\mathbf{D})$ the maximum of dual problem \eqref{equ: problem (D)}. Let
$\Delta_i:=\sup_{\w\in\cW_i}\left\{\widehat{f}_i(\w)-\widetilde{f}_i(\w)\right\}\ge 0$
and $\Delta_{worst}:=\max_{i\in[I]} \Delta_i$.
Suppose $\cW_i$'s are compact and both $f_i(\w_{(i)};\x)$ and $h_i(\w_{(i)})$ are continuous w.r.t. $\w_{(i)}$. If there exists at least one feasible solution of problem \textup{(\textbf{P})}, then under Assumption \ref{assumption: bounded margin} the duality gap w.r.t. problems \eqref{equ: problem (P)} and \eqref{equ: problem (D)} can be bounded by
\begin{equation*}
0\le \frac{\inf(\mathbf{P})-\sup(\mathbf{D})}{\Delta_{worst}}\le \frac{2}{I}.
\end{equation*}
\end{theorem}

Note that $\Delta_i$ measures the divergence between the function value of $\widehat f_i$ and its convex relaxation $\widetilde f_i$. The constant $\Delta_{worst}$ is the maximal divergence among all sub-networks, which grows slowly with the increase of $I$. This is because $\Delta_{worst}$ only measures the divergence of \emph{one branch}. The normalized duality gap $(\inf(\mathbf{P})-\sup(\mathbf{D}))/\Delta_{worst}$ has been widely used before to measure the degree of non-convexity of optimization problems~\cite{bertsekas1982estimates,udell2016bounding,bi2016refined,fang2015blessing,d2017approximate}. Such a normalization avoids trivialities in characterizing the degree of non-convexity: scaling the objective function by any constant does not change the value of normalized duality gap.
Even though Theorem~\ref{theorem: deep non-linear neural networks} is in the form of population risk, the conclusion still holds for the \emph{empirical loss} as well. This can be achieved by setting the marginal distribution $\cP_{\x}$ as the uniform distribution on a finite set and $\cP_y$ as the corresponding labels uniformly distributed on the same finite set.

\medskip
\noindent{\textbf{Inspiration for Architecture Designs.}} Theorem \ref{theorem: deep non-linear neural networks} shows that the loss surface of deep network is less non-convex when the width $I$ is large; when $I\rightarrow +\infty$, surprisingly, deep network is as easy as a convex optimization. An intuitive explanation is that the large number of randomly initialized hidden units represent all possible features. Thus the optimization problem involves just training the top layer of the network, which is convex.
Our result encourages a class of network architectures with multiple branches and supports some of the most successful architectures in practice, such as Inception~\cite{szegedy2017inception}, Xception~\cite{chollet2016xception}, ResNeXt~\cite{xie2017aggregated}, SqueezeNet~\cite{iandola2016squeezenet}, Wide ResNet~\cite{zagoruyko2016wide}, Shake-Shake regularization~\cite{gastaldi2017shake} --- all of which benefit from the split-transform-merge behaviour as shown in Figure~\ref{figure: architecture}. The theory sheds light on an explanation of strong performance of these architectures.

\medskip
\noindent{\textbf{Related Works.}}
While many efforts have been devoted to studying the local minima or saddle points of deep neural networks~\cite{liang2018adding,zhou2018empirical,soudry2016no,kawaguchi2017deep,xie2017diverse,rene2017math}, little is known about the duality gap of deep networks. In particular, Choromanska et al.~\cite{choromanska2015open,choromanska2015loss} showed that the number of poor local minima
cannot be too large. Kawaguchi~\cite{kawaguchi2016deep} improved over the results of \cite{choromanska2015open,choromanska2015loss} by assuming that the activation functions are independent Bernoulli variables and the input data are drawn from Gaussian distribution. Xie et al.~\cite{xie2017diverse} and Haeffele et al.~\cite{haeffele2015global} studied the local minima of regularized network, but they require either the network is shallow, or the network weights are rank-deficient. Ge et al.~\cite{ge2017learning} showed that every local minimum is globally optimal by modifying the activation function. Zhang et al.~\cite{zhang2016convexified} and Aslan et al.~\cite{aslan2014convex} reduced the non-linear activation to the linear case by kernelization and relaxed the non-convex problem to a convex one. However, no formal guarantee was provided for the tightness of the relaxation. Theorem \ref{theorem: deep non-linear neural networks}, on the other hand, bounds the duality gap of deep neural networks \emph{with mild assumptions}.

Another line of research studies the convexity behaviour of neural networks when the number of hidden neurons goes to the infinity. In particular, Bach~\cite{JMLR:v18:14-546} proved that a single-hidden-layer network is as easy as a convex optimization by using classical non-Euclidean regularization tools. Bengio et al.~\cite{bengio2006convex} showed a similar phenomenon for multi-layer networks with an incremental algorithm. In comparison, Theorem  \ref{theorem: deep non-linear neural networks} not only captures the convexification phenomenon when $I\rightarrow +\infty$, but also goes beyond the result as it characterizes the convergence rate of convexity of neural networks in terms of duality gap. Furthermore, the conclusion in Theorem \ref{theorem: deep non-linear neural networks} holds for the population risk, which was unknown before.

\section{Strong Duality of Linear Neural Networks}

In this section, we show that the duality gap is zero if the activation function is linear.
Deep linear neural network has received significant attention in recent years~\cite{saxe2013exact,kawaguchi2016deep,zhang2016convexified,lu2017depth,baldi1989neural,Goodfellow-et-al-2016,hardt2016identity,baldi2012complex} because of its simple formulation\footnote{Although the expressive power of deep linear neural networks and three-layer linear neural networks are the same, the analysis of landscapes of two models are significantly different, as pointed out by \cite{Goodfellow-et-al-2016,kawaguchi2016deep,lu2017depth}.} and its connection to non-linear neural networks.

\medskip
\noindent{\textbf{Network Setup.}}
We discuss the strong duality of regularized deep linear neural networks of the form
\begin{equation}
\label{equ: regularized PCA}
(\W_1^*,...,\W_H^*)
= \argmin_{\W_1,...,\W_H} \frac{1}{2}\|\Y-\W_H\cdots\W_1\X\|_F^2+\frac{\gamma}{H}\left[ \|\W_1\X\|_{\cS_H}^H + \sum_{i=2}^{H}\|\W_i\|_{\cS_H}^H \right],
\end{equation}
where $\X=[\x_1,...,\x_n]\in\R^{d_0\times n}$ is the given instance matrix, $\Y=[\y_1,...,\y_n]\in\R^{d_H\times n}$ is the given label matrix, and $\W_i\in\R^{d_i\times d_{i-1}},i\in[I]$ represents the weight matrix in each linear layer.
We mention that (a) while the linear operation is simple matrix multiplications in problem \eqref{equ: regularized PCA}, it can be easily extended to other linear operators, e.g., the convolutional operator or the linear operator with the bias term, by properly involving a group of kernels in the variable $\W_i$~\cite{haeffele2015global}. (b) The regularization terms in problem \eqref{equ: regularized PCA} are of common interest, e.g., see \cite{haeffele2015global}. When $H=2$, our regularization terms reduce to $\frac{1}{2}\|\W_i\|_F^2$, which is well known as the weight-decay or Tikhonov regularization. (c) The regularization parameter $\gamma$ is the same for each layer since we have no further information on the preference of layers.

Our analysis leads to the following guarantees for the deep linear neural networks.
\begin{theorem}
\label{theorem: local=global}
Denote by $\tY:=\Y\X^\dag\X\in\R^{d_H\times n}$ and $d_{\min}:=\min\{d_1,...,d_{H-1}\}\le \min\{d_0,d_H,n\}$. Let $0\le\gamma<\sigma_{\min}(\tY)$ and $H\ge 2$, where $\sigma_{\min}(\tY)$ stands for the minimal non-zero singular value of $\tY$.
Then the strong duality holds for deep linear neural network \eqref{equ: regularized PCA}. In other words, the optimum of problem \eqref{equ: regularized PCA} is the same as its convex dual problem
\begin{equation}
\label{equ: dual problem}
\bLambda^*=\argmax_{\mathsf{Row}(\bLambda)\subseteq\mathsf{Row}(\X)} -\frac{1}{2}\|\tY-\bLambda\|_{d_{\min}}^2+\frac{1}{2}\|\Y\|_F^2,\quad\textup{s.t.}\quad \|\bLambda\|\le\gamma,
\end{equation}
where $\|\cdot\|_{d_{\min}}^2=\sum_{i=1}^{d_{\min}}\sigma_i^2(\cdot)$ is a convex function. Moreover, the optimal solutions of primal problem \eqref{equ: regularized PCA} can be obtained from the dual problem \eqref{equ: dual problem} in the following way: let $\U\bSigma\V^T=\textup{\textsf{svd}}_{d_{\min}}(\tY-\bLambda^*)$ be the skinny SVD of matrix $\textup{\textsf{svd}}_{d_{\min}}(\tY-\bLambda^*)$, then $\W_i^*=[\bSigma^{1/H},\0;\0,\0]\in\R^{d_i\times d_{i-1}}$ for $i=2,3,...,H-1$, $\W_H^*=[\U\bSigma^{1/H},\0]\in\R^{d_H\times d_{H-2}}$ and $\W_1^*=[\bSigma^{1/H}\V^T;\0]\X^\dag\in\R^{d_1\times d_0}$ is a globally optimal solution to problem \eqref{equ: regularized PCA}.
\end{theorem}

The regularization parameter $\gamma$ cannot be too large in order to avoid underfitting. Our result provides a suggested upper bound $\sigma_{\min}(\tY)$ for the regularization parameter, where oftentimes $\sigma_{\min}(\tY)$ characterizes the level of random noise. When $\gamma=0$, our analysis reduces to the \emph{un-regularized deep linear neural network}, a model which has been widely studied in \cite{kawaguchi2016deep,lu2017depth,baldi1989neural,Goodfellow-et-al-2016}.

Theorem \ref{theorem: local=global} implies the followig result on the landscape of deep linear neural networks: the regularized deep learning can be converted into an equivalent convex problem by dual. We note that the strong duality rarely happens in the non-convex optimization: matrix completion~\cite{Balcan2017optimal}, Fantope~\cite{overton1992sum}, and quadratic optimization with two quadratic constraints~\cite{beck2006strong} are among the few paradigms that enjoy the strong duality.
For deep networks, the effectiveness of convex relaxation has been observed empirically in~\cite{aslan2014convex,zhang2016convexified}, but much remains unknown for the theoretical guarantees of the relaxation. Our work shows strong duality of regularized deep linear neural networks and provides an alternative approach to overcome the computational obstacles due to the non-convexity: one can apply convex solvers, e.g., the Douglas–Rachford algorithm,\footnote{Grussler et al.~\cite{grussler2016low} provided a fast algorithm to compute the proximal operators of $\frac{1}{2}\|\cdot\|_{d_{\min}}^2$. Hence, the Douglas–Rachford algorithm can find the global solution up to an $\epsilon$ error in function value in time $\textsf{poly}(1/\epsilon)$~\cite{he20121}.} for problem \eqref{equ: dual problem} and then conduct singular value decomposition to compute the weights $\{\W_i^*\}_{i=1}^H$ from $\textup{\textsf{svd}}_{d_{\min}}(\tY-\bLambda^*)$.
In addition, our result inherits the benefits of convex analysis. The vast majority results on deep learning study the generalization error or expressive power by analyzing its complicated non-convex form~\cite{neyshabur2015path,zhang2017learnability,zhang2016understanding}. In contrast, with strong duality one can investigate various properties of deep linear networks with much simpler convex form.

\medskip
\noindent{\textbf{Related Works.}}
The goal of convexified linear neural networks is to relax the non-convex form of deep learning to the computable convex formulations~\cite{zhang2016convexified,aslan2014convex}. While several efforts have been devoted to investigating the effectiveness of such convex surrogates, e.g., by analyzing the generalization error after the relaxation~\cite{zhang2016convexified}, little is known whether the relaxation is tight to its original problem. Our result, on the other hand, provides theoretical guarantees for the tightness of convex relaxation of deep linear networks, a phenomenon observed empirically in~\cite{aslan2014convex,zhang2016convexified}.

We mention another related line of research --- no bad local minima. On one hand, although recent works have shown the absence of spurious local minimum for deep linear neural networks~\cite{Saxe2015Deep,kawaguchi2016deep,lu2017depth}, many of them typically lack theoretical analysis of regularization term. Specifically, Kawaguchi~\cite{kawaguchi2016deep} showed that \emph{un-regularized} deep linear neural networks have no spurious local minimum. Lu and Kawaguchi~\cite{lu2017depth} proved that depth creates no bad local minimum for \emph{un-regularized} deep linear neural networks. In contrast, our optimization problem is more general by taking the regularization term into account. On the other hand, even the ``local=global'' argument holds for the deep linear neural networks, it is still hard to escape bad saddle points~\cite{anandkumar2016efficient}. In particular, Kawaguchi~\cite{kawaguchi2016deep} proved that for linear networks deeper than three layers, there exist bad saddle points at which the Hessian does not have any negative eigenvalue. Therefore, the state-of-the-art algorithms designed to escape the saddle points might not be applicable~\cite{jin2017escape,ge2015escaping}.
Our result provides an alternative approach to solve deep linear network by convex programming, which bypasses the computational issues incurred by the bad saddle points.

\comment{
\begin{table}
\caption{Study of linear neural networks.}
\label{table: comparison of sample complexity}
\centering
\begin{tabular}{c|ccc}%
\hline
& Convexification & No Spurious Local Minima & Strong Duality\\
\hline
Literature & \cite{aslan2014convex,zhang2016convexified} & \cite{baldi1989neural,kawaguchi2016deep,lu2017depth} & Ours\\
\hline
\end{tabular}
\end{table}
}

\vspace{-0.2cm}
\section{Our Techniques and Proof Sketches}
\vspace{-0.1cm}
\label{section: Our Techniques and Proof Sketches}

In this section, we present our techniques and proof sketches of Theorems \ref{theorem: deep non-linear neural networks} and \ref{theorem: local=global}.

\medskip
\noindent{\textbf{(a) Shapley-Folkman Lemma.}} The proof of Theorem \ref{theorem: deep non-linear neural networks} is built upon the Shapley-Folkman lemma~\cite{d2017approximate,starr1969quasi,fang2015blessing,bertsekas1982estimates}, which characterizes a convexification phenomenon concerning the average of multiple sets and is analogous to the central limit theorem in the probability theory. Consider the averaged Minkowski sum of $I$ sets $\cA_1,\cA_2,...,\cA_I$ given by $\{I^{-1}\sum_{j\in[I]}a_j:a_j\in\cA_j\}$. Intuitively, the lemma states that $\rho(I^{-1}\sum_{j\in[I]} \cA_j)\rightarrow 0$ as $I\rightarrow +\infty$, where $\rho(\cdot)$ is a metric of the non-convexity of a set (see Figure \ref{figure: Shapley-Folkman} for visualization). We apply this lemma to the optimization formulation of deep neural networks. Denote by \emph{augmented epigraph} the set $\{(h(\w),\ell(\w)):\text{all possible choices of }\w\}$, where $h$ is the constraint and $\ell$ is the objective function in the optimization problem. The key observation is that the augmented epigraph of neural network loss with multi-branch architecture can be expressed as the Minkowski average of augmented epigraphs of all branches. Thus we obtain a natural connection between an optimization problem and its corresponding augmented epigraph. Applying Shapley-Folkman lemma to the augmented epigraph leads to a characteristic of non-convexity of the deep neural network.

\begin{figure}[t]
\center
\includegraphics[width=0.9\textwidth]{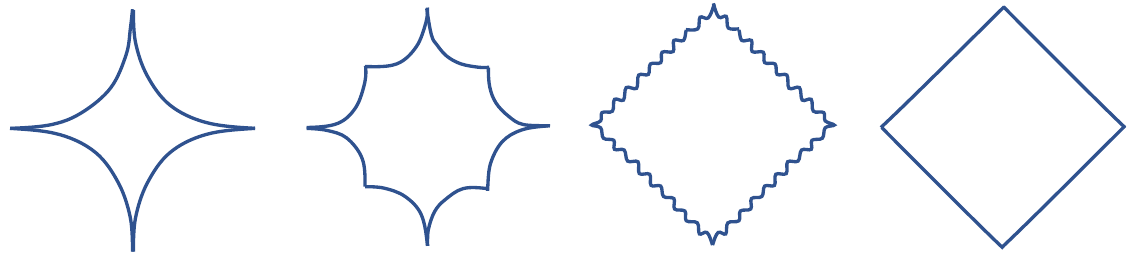}
\caption{Visualization of Shapley-Folkman lemma. \textbf{The first figure:} an $\ell_{1/2}$ ball. \textbf{The second and third figures:} the averaged Minkowski sum of two and ten $\ell_{1/2}$ balls. \textbf{The fourth figure:} the convex hull of $\ell_{1/2}$ ball (the Minkowski average of infinitely many $\ell_{1/2}$ balls). It show that with the number of $\ell_{1/2}$ balls to be averaged increasing, the Minkowski average tends to be more convex.}
\label{figure: Shapley-Folkman}
\end{figure}

\medskip
\noindent{\textbf{(b) Variational Form.}} The proof of Theorem \ref{theorem: local=global} is built upon techniques (b), (c), and (d). In particular, problem \eqref{equ: regularized PCA} is highly non-convex due to its multi-linear form over the optimized variables $\{\W_i\}_{i=1}^H$. Fortunately, we are able to analyze the problem by grouping $\W_H\W_{H-1}...\W_1\X$ together and converting the original non-convex problem in terms of the separate variables $\{\W_i\}_{i=1}^H$ to a convex optimization with respect to the new grouping variable $\W_H\W_{H-1}...\W_1\X$. This typically requires us to represent the objective function of \eqref{equ: regularized PCA} as a convex function of $\W_H\W_{H-1}...\W_1$. To this end, we prove that $\|\W_H\W_{H-1}...\W_1\X\|_*=\min_{\W_1,...,\W_H} \displaystyle\frac{1}{H} \left[ \|\W_1\X\|_{\cS_H}^H + \sum_{i=2}^{H}\|\W_i\|_{\cS_H}^H \right]$ (see Lemma \ref{lemma: variational form of nuclear norm} in Appendix \ref{section: Strong Duality of Deep Linear Neural Networks}). So the objective function in problem \eqref{equ: regularized PCA} has an equivalent form
\begin{equation}
\label{equ: nuclear norm form technique}
\min_{\W_1,...,\W_H} \frac{1}{2}\|\Y-\W_H\W_{H-1}\cdots\W_1\X\|_F^2+\gamma\|\W_H\W_{H-1}\cdots\W_1\X\|_*.
\end{equation}
This observation enables us to represent the optimization problem as a convex function of the output of a neural network. Therefore, we can analyze the non-convex problem by applying powerful tools from convex analysis.

\medskip
\noindent{\textbf{(c) Reduction to Low-Rank Approximation.}} Our results of strong duality concerning problem \eqref{equ: nuclear norm form technique} are inspired by the problem of low-rank matrix approximation:
\begin{equation}
\label{equ: PCA}
\min_{\W_1,...,\W_H} \frac{1}{2}\|\Y-\bLambda^*-\W_H\W_{H-1}\cdots\W_1\X\|_F^2.
\end{equation}
We know that all local solutions of \eqref{equ: PCA} are globally optimal~\cite{kawaguchi2016deep,lu2017depth,Balcan2017optimal}. To analyze the more general regularized problem \eqref{equ: regularized PCA}, our main idea is to reduce problem \eqref{equ: nuclear norm form technique} to the form of \eqref{equ: PCA} by Lagrangian function. In other words, the Lagrangian function of problem \eqref{equ: nuclear norm form technique} should be of the form \eqref{equ: PCA} for a fixed Lagrangian variable $\bLambda^*$, which we will construct later in subsection~(d). While some prior works attempted to apply a similar reduction, their conclusions either depended on unrealistic conditions on local solutions, e.g., all local solutions are rank-deficient~\cite{haeffele2015global,grussler2016low}, or their conclusions relied on strong assumptions on the objective functions, e.g., that the objective functions are twice-differentiable~\cite{haeffele2015global}, which do not apply to the non-smooth problem \eqref{equ: nuclear norm form technique}. Instead, our results bypass these obstacles by formulating the strong duality of problem \eqref{equ: nuclear norm form technique} as the existence of a dual certificate $\bLambda^*$ satisfying certain dual conditions (see Lemma \ref{lemma: f and k local minimum} in Appendix \ref{section: Strong Duality of Deep Linear Neural Networks}). Roughly, the dual conditions state that the optimal solution $(\W_1^*,\W_2^*,...,\W_H^*)$ of problem \eqref{equ: nuclear norm form technique} is locally optimal to problem \eqref{equ: PCA}. On one hand, by the above-mentioned properties of problem \eqref{equ: PCA}, $(\W_1^*,...,\W_H^*)$ globally minimizes the Lagrangian function when $\bLambda$ is fixed to $\bLambda^*$. On the other hand, by the convexity of nuclear norm, for the fixed $(\W_1^*,...,\W_H^*)$ the Lagrangian variable $\bLambda^*$ globally optimize the Lagrangian function. Thus $(\W_1^*,...,\W_H^*,\bLambda^*)$ is a primal-dual saddle point of the Lagrangian function of problem \eqref{equ: nuclear norm form technique}. The desired strong duality is a straightforward result from this argument.

\vspace{-0.1cm}
\medskip
\noindent{\textbf{(d) Dual Certificate.}} The remaining proof is to construct a dual certificate $\bLambda^*$ such that the dual conditions hold true. The challenge is that the dual conditions impose several constraints simultaneously on the dual certificate (see condition \eqref{equ: dual conditions for general problem} in Appendix \ref{section: Strong Duality of Deep Linear Neural Networks}), making it hard to find a desired certificate. This is why progress on the dual certificate has focused on convex programming.
To resolve the issue, we carefully choose the certificate as an appropriate scaling of subgradient of nuclear norm around a low-rank solution, where the nuclear norm follows from our regularization term in technique (b). Although the nuclear norm has infinitely many subgradients, we prove that our construction of dual certificate obeys all desired dual conditions. Putting techniques (b), (c), and (d) together, our proof of strong duality is completed.

\begin{figure}[t]
  \begin{subfigure}{0.24\textwidth}
    \includegraphics[width=\textwidth]{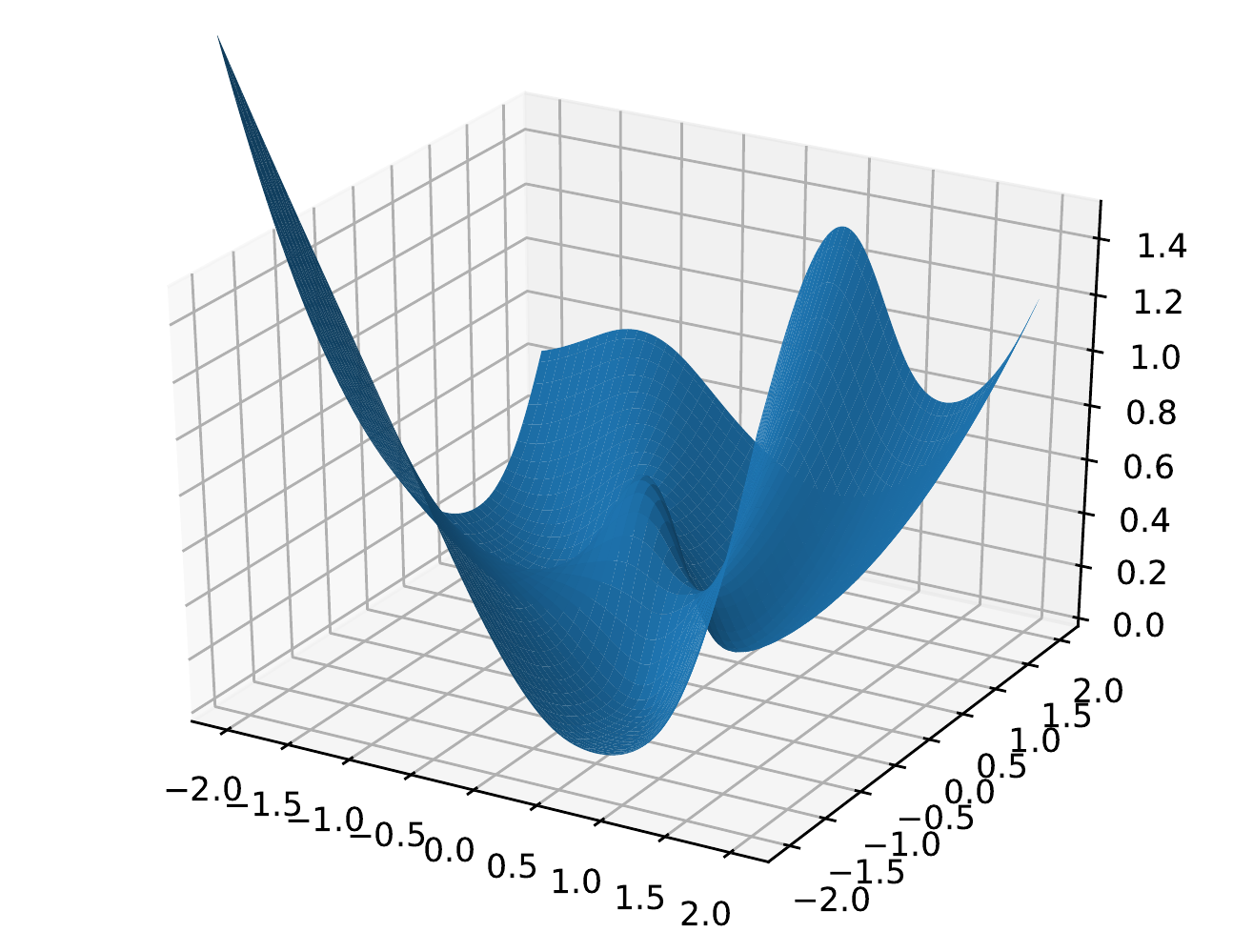}
    \caption{$I=1$.}
  \end{subfigure}
  %
  %
  \begin{subfigure}{0.24\textwidth}
    \includegraphics[width=\textwidth]{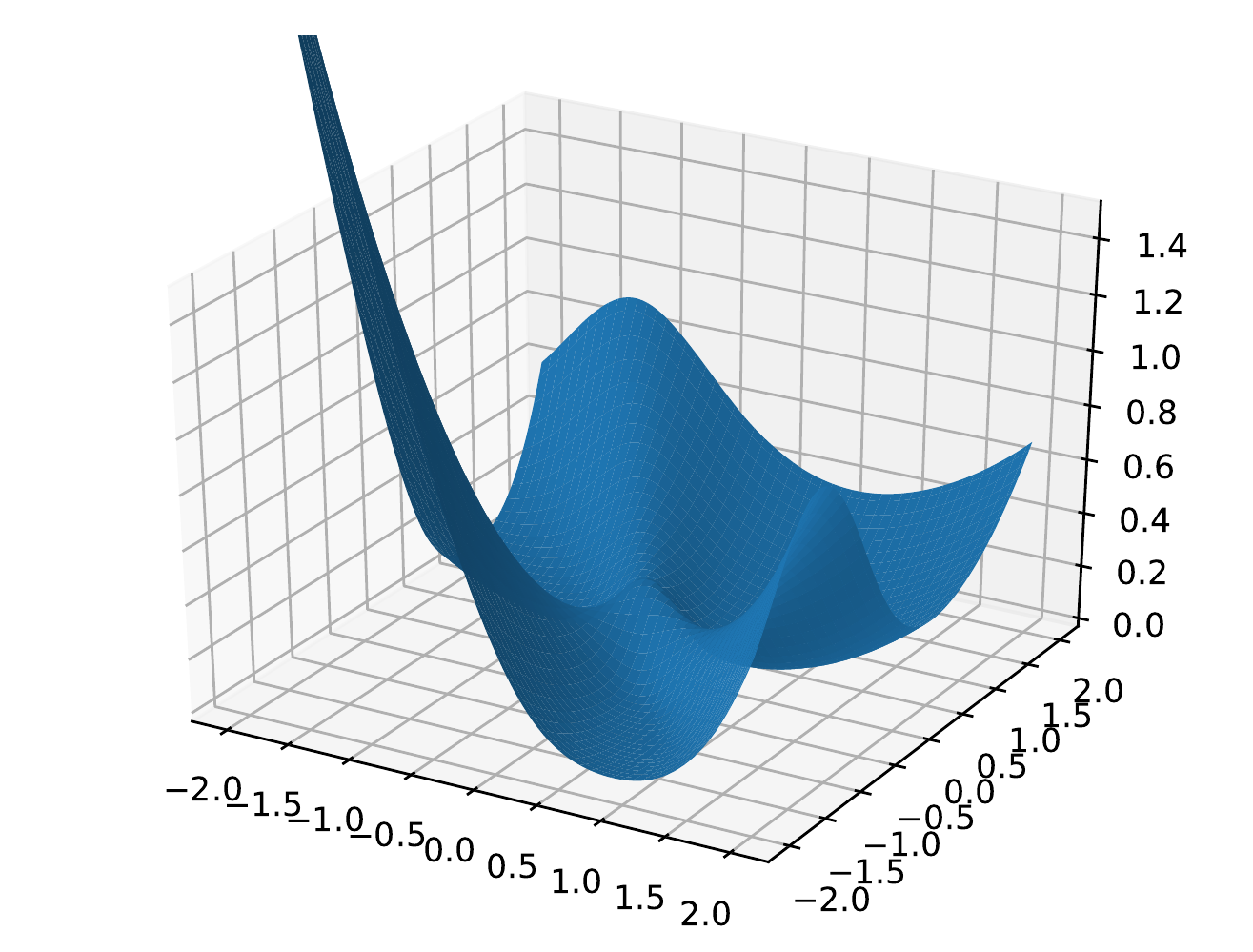}
    \caption{$I=3$.}
  \end{subfigure}
  %
  \begin{subfigure}{0.24\textwidth}
    \includegraphics[width=\textwidth]{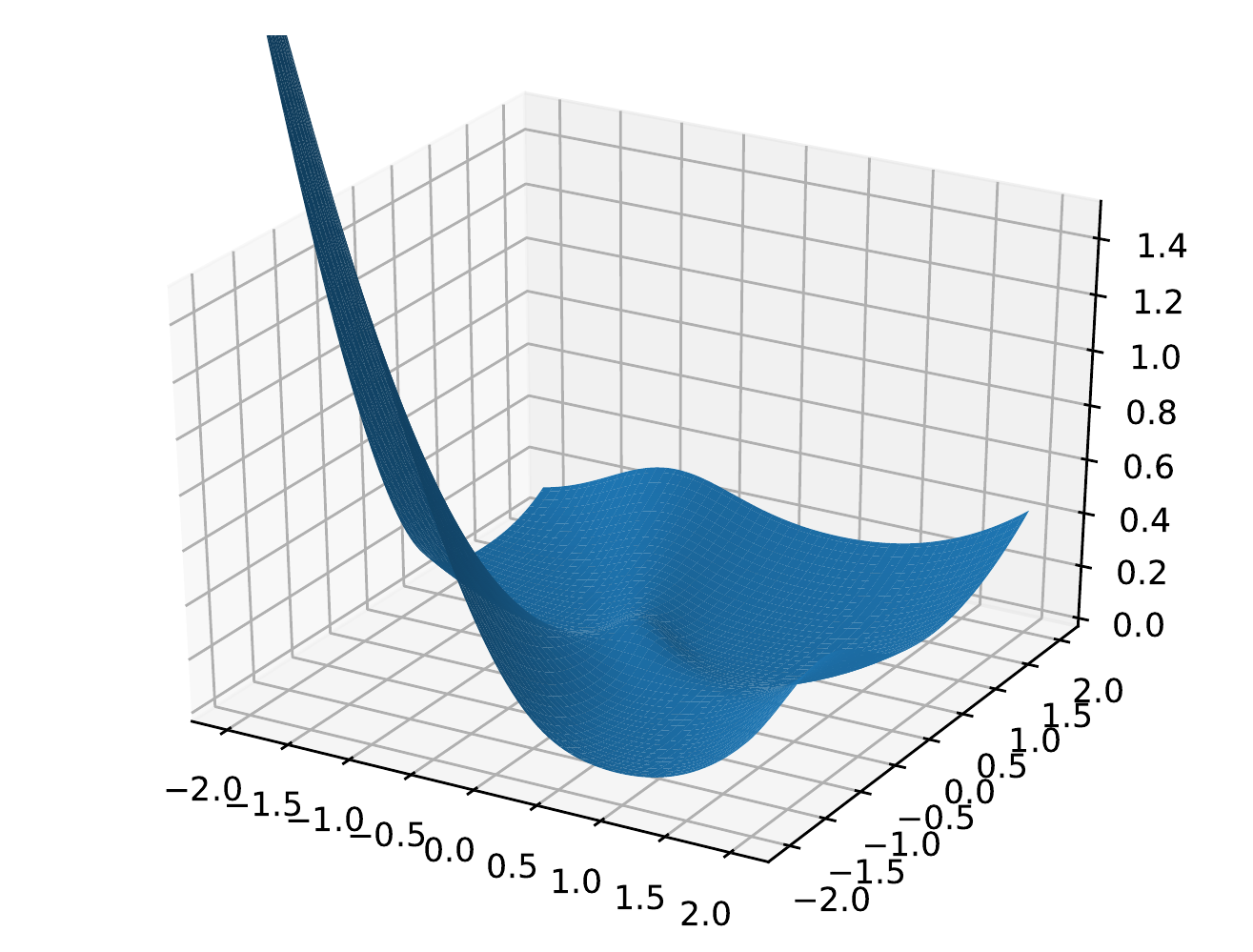}
    \caption{$I=5$.}
  \end{subfigure}
  %
  %
  %
  \begin{subfigure}{0.24\textwidth}
    \includegraphics[width=\textwidth]{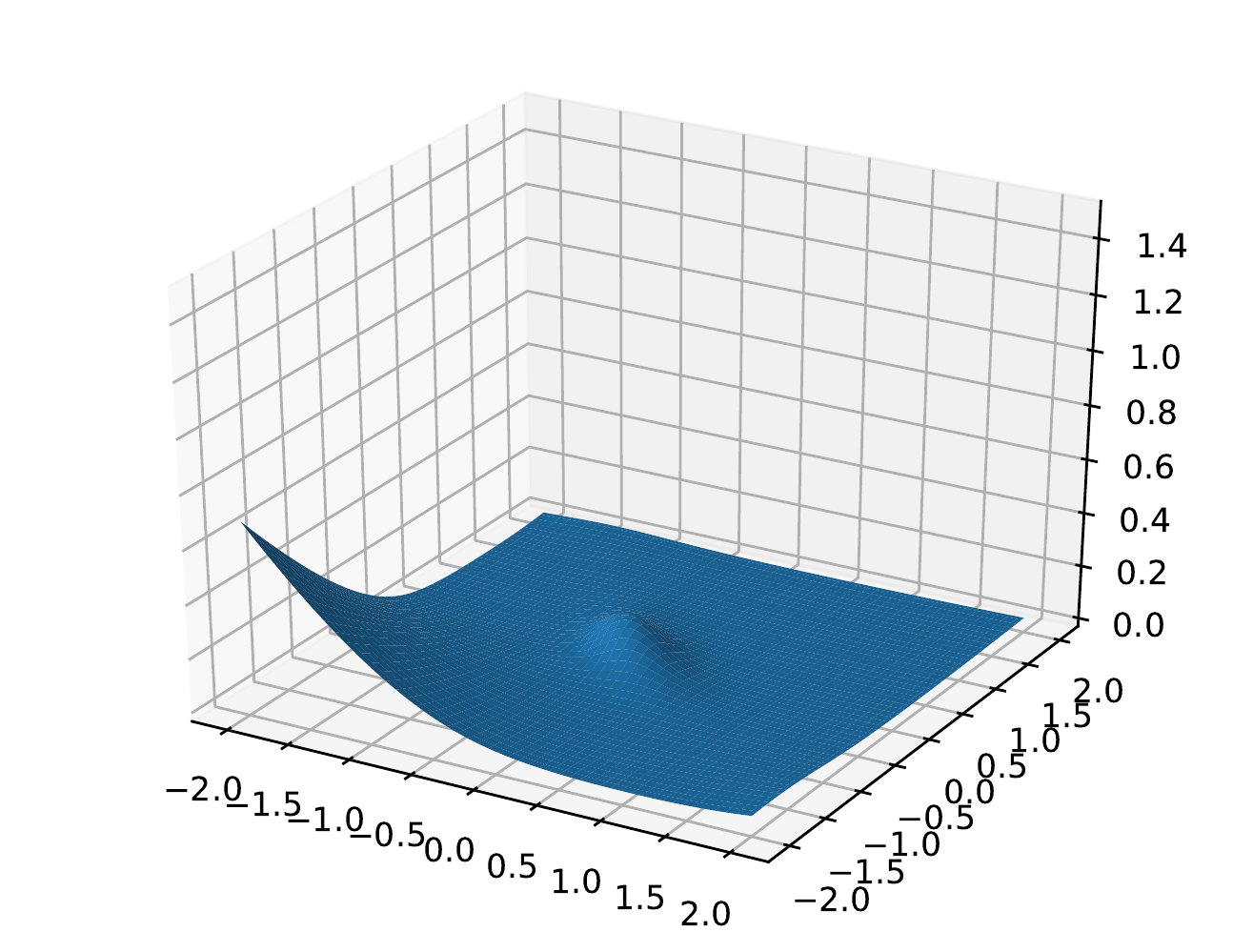}
    \caption{$I=100$.}
  \end{subfigure}
  \begin{subfigure}{0.24\textwidth}
    \includegraphics[width=\textwidth]{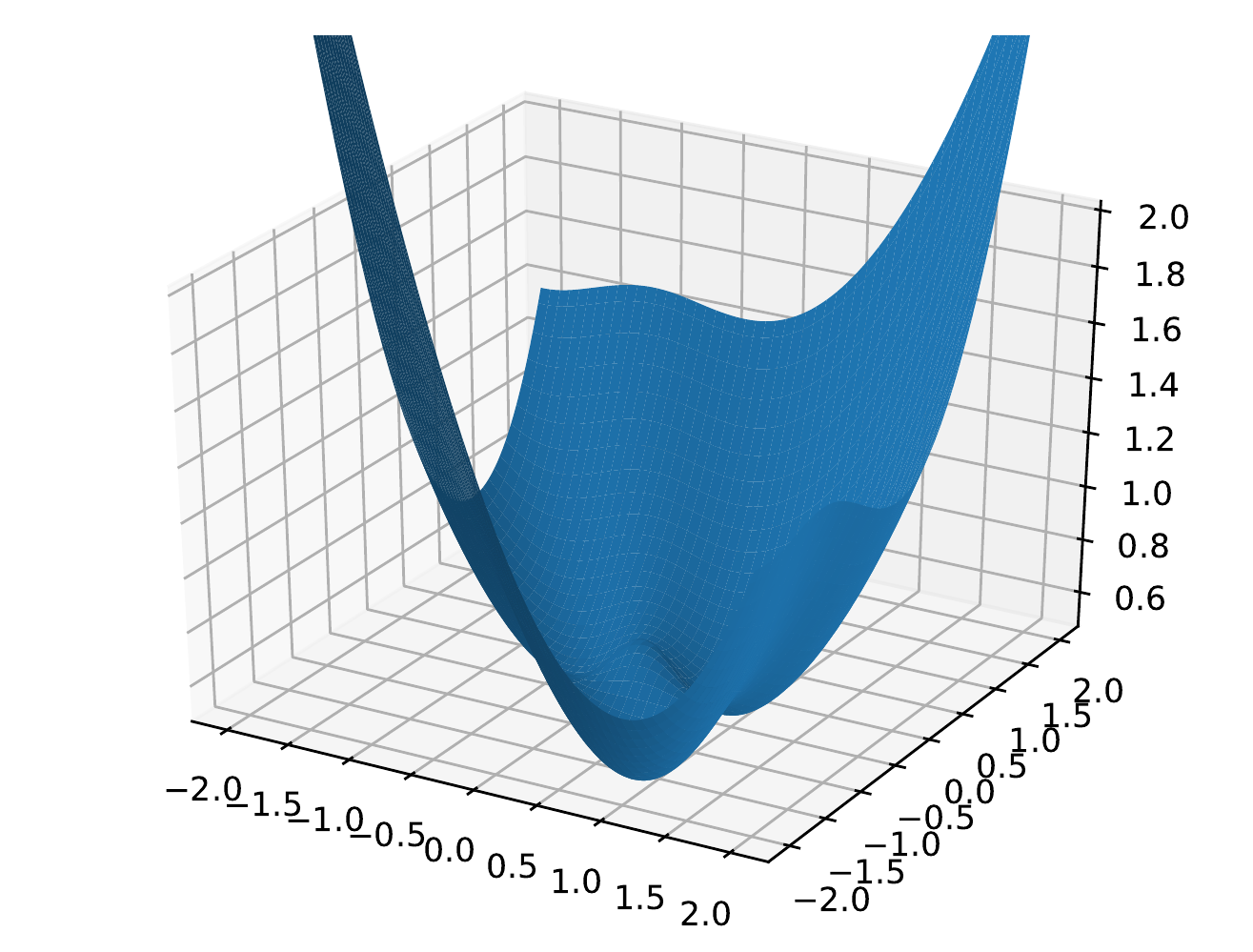}
    \caption{$I=1$.}
  \end{subfigure}
  \begin{subfigure}{0.24\textwidth}
    \includegraphics[width=\textwidth]{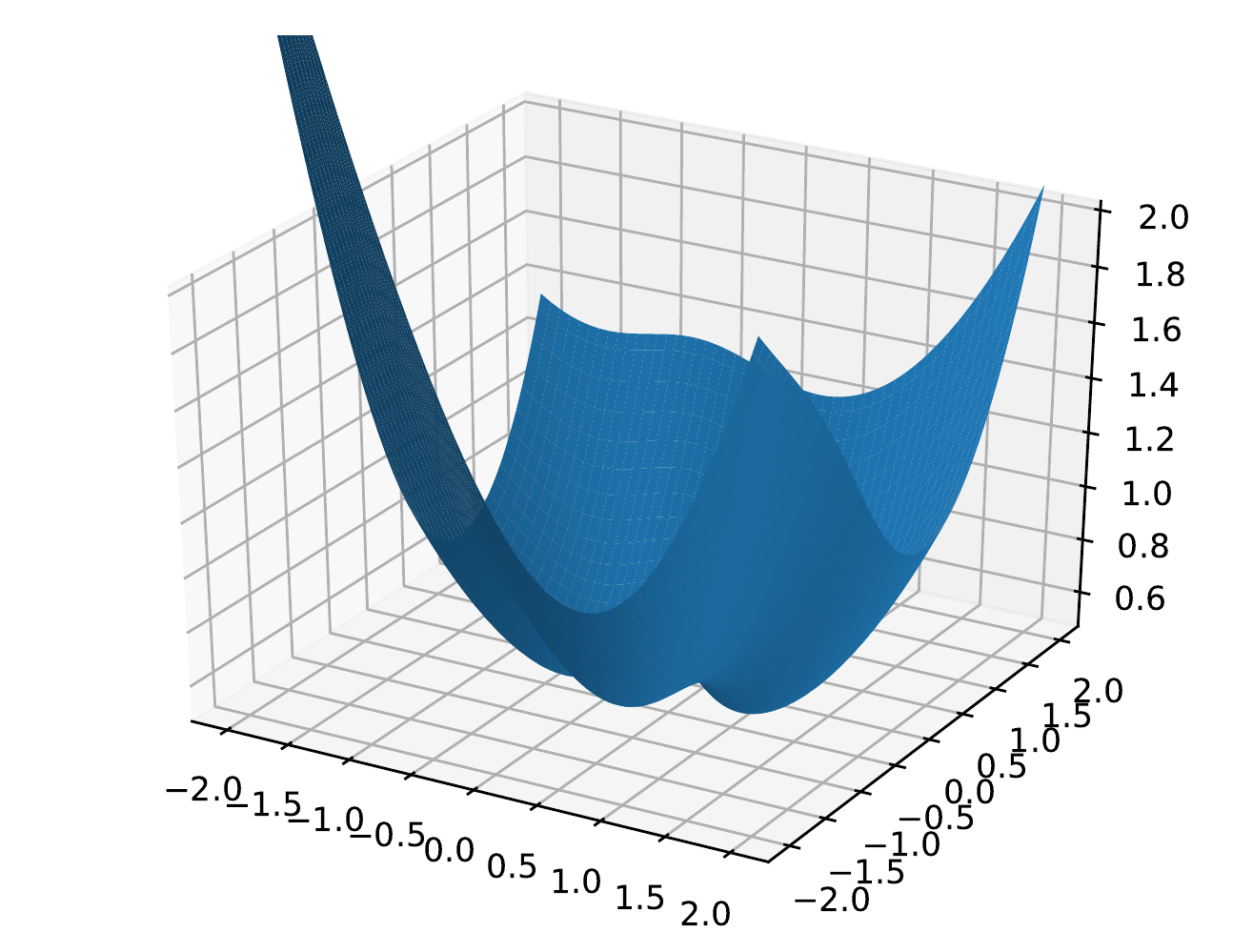}
    \caption{$I=3$.}
  \end{subfigure}
  \begin{subfigure}{0.24\textwidth}
    \includegraphics[width=\textwidth]{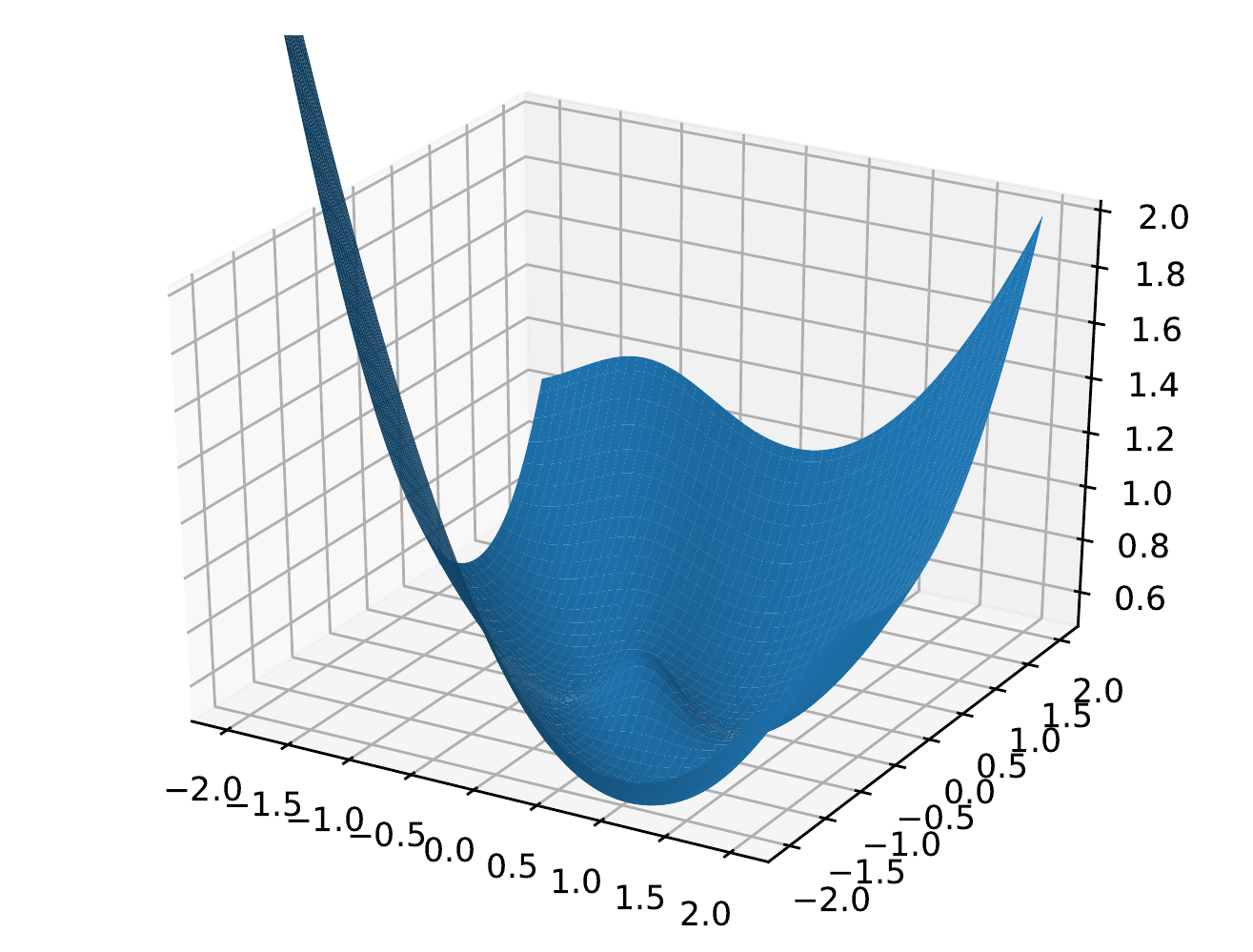}
    \caption{$I=5$.}
  \end{subfigure}
  \begin{subfigure}{0.24\textwidth}
    \includegraphics[width=\textwidth]{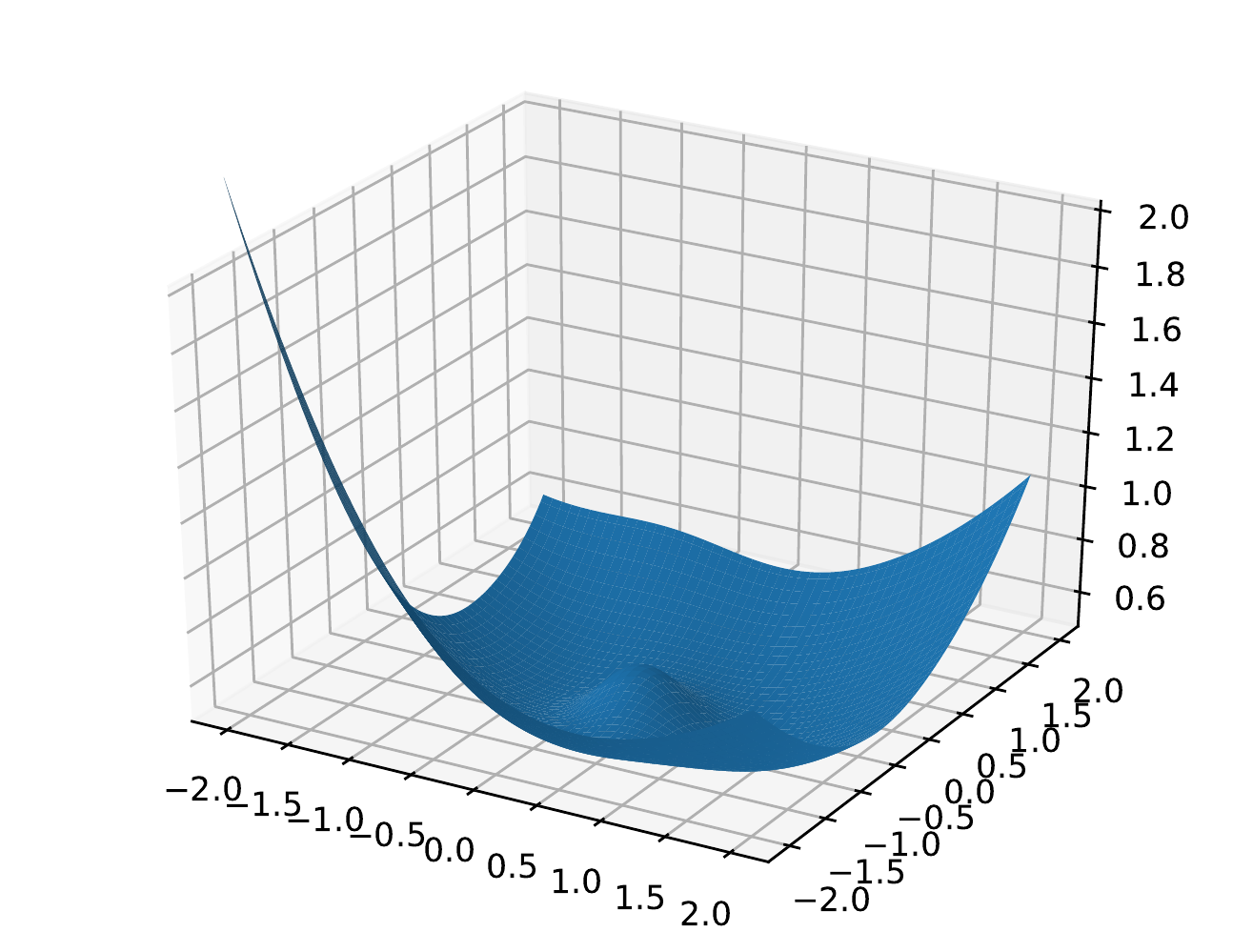}
    \caption{$I=100$.}
  \end{subfigure}
  \begin{subfigure}{0.24\textwidth}
    \includegraphics[width=\textwidth]{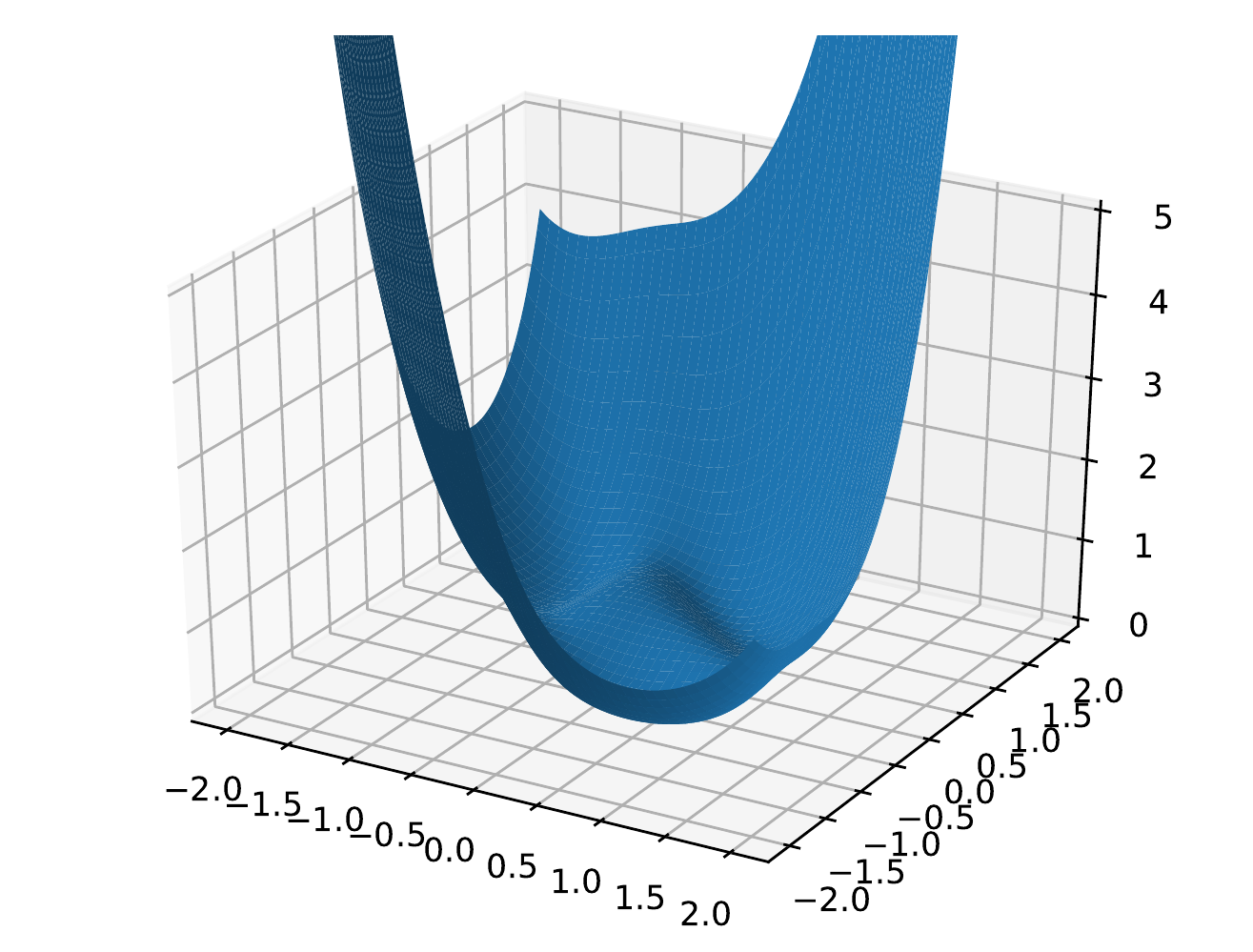}
    \caption{$I=1$.}
  \end{subfigure}
  %
  %
  \begin{subfigure}{0.24\textwidth}
    \includegraphics[width=\textwidth]{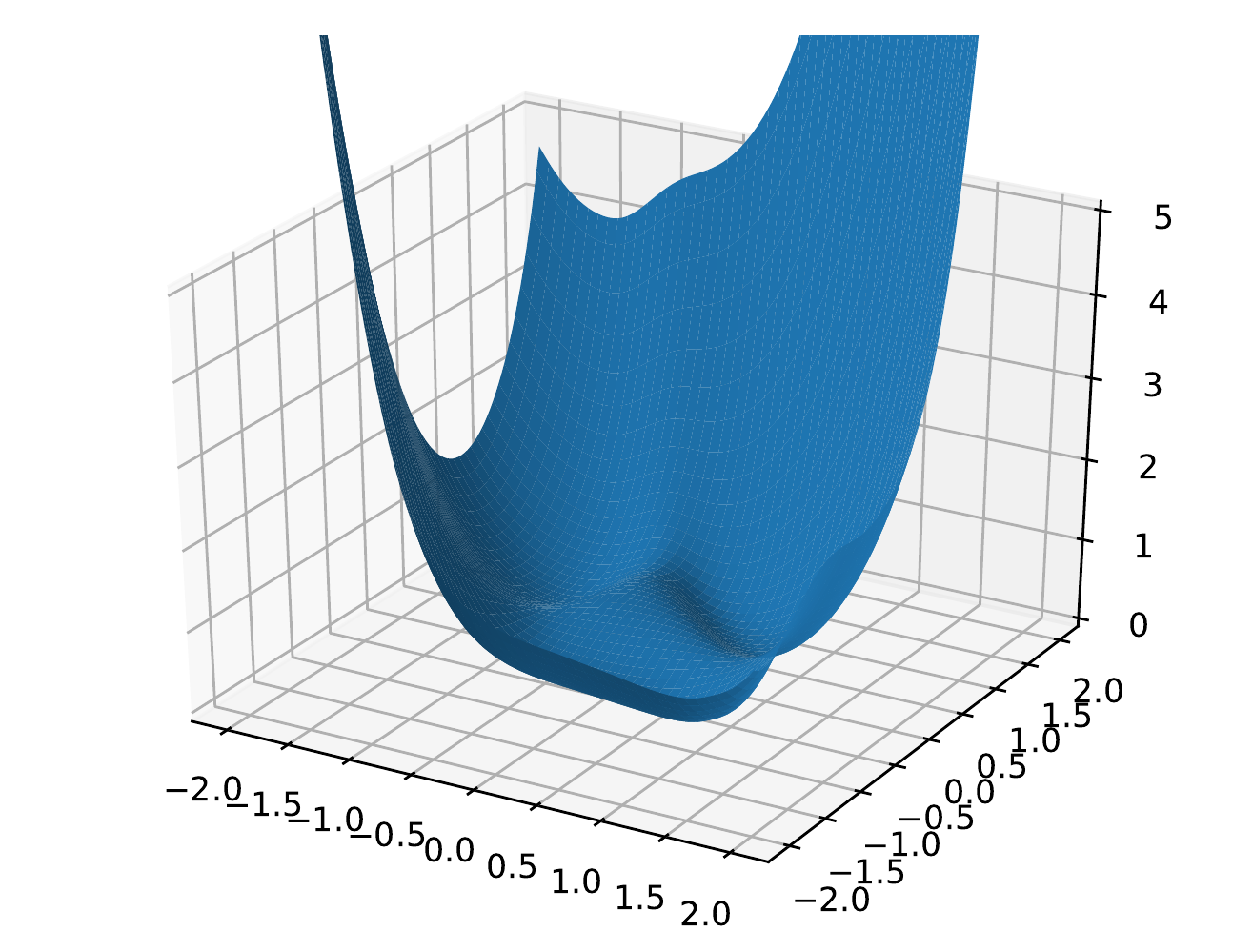}
    \caption{$I=3$.}
  \end{subfigure}
  %
  \begin{subfigure}{0.24\textwidth}
    \includegraphics[width=\textwidth]{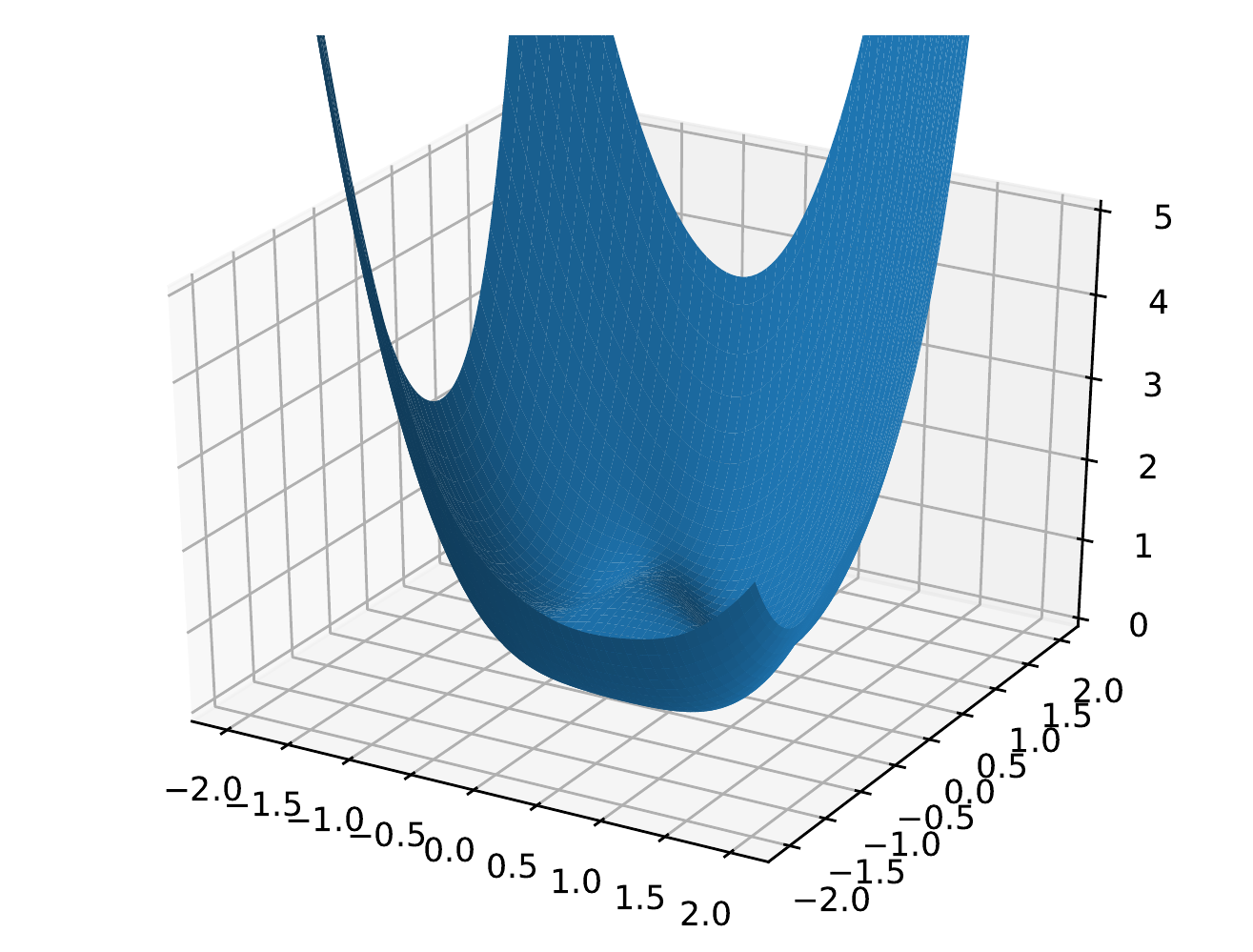}
    \caption{$I=5$.}
  \end{subfigure}
  %
  %
  %
  \begin{subfigure}{0.24\textwidth}
    \includegraphics[width=\textwidth]{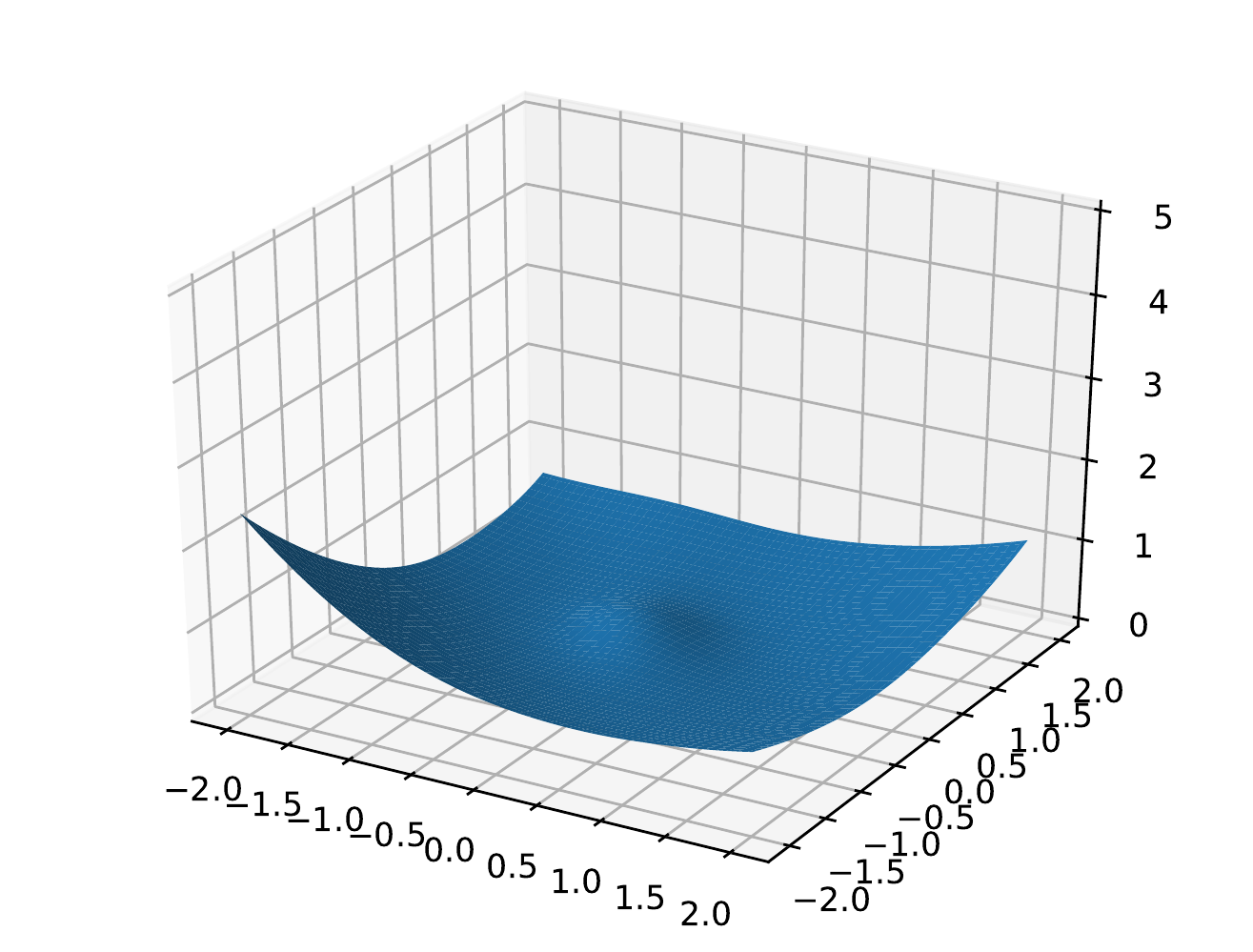}
    \caption{$I=50,000$.}
  \end{subfigure}

  \caption{\textbf{Top Row:} Landscape of one-hidden-layer network on MNIST. \textbf{Middle Row:} Landscape of one-hidden-layer network on CIFAR-10. \textbf{Bottom Row:} Landscape of three-hidden-layer, multi-branch network on CIFAR-10 dataset. From left to right, the landscape looks less non-convex.}
  \label{figure: experiment B}
\vspace{-0.3cm}
\end{figure}

\section{Experiments}
\label{section: experiments}
In this section, we verify our theoretical contributions by the experimental validation. We release our PyTorch code at \url{https://github.com/hongyanz/multibranch}.
\comment{
\begin{table}
\caption{Overview of the datasets}
  \begin{center}
    \begin{tabular}{c||c|c|c}
    \hline
    & Training instances & Shape of data sample & Task\tabularnewline
    \hline
    \hline
    Synthetic & 1000 & 10 & Binary classification\tabularnewline
    \hline
    MNIST & 60000 & $28\times28$ & 10-class classification\tabularnewline
    \hline
    CIFAR-10 & 50000 & $3\times32\times32$ & 10-class classification\tabularnewline
    \hline
    \end{tabular}
  \end{center}
  \label{table: overview of datasets}
\end{table}
}

\subsection{Visualization of Loss Landscape}

\label{section: Visualization of Loss Landscape}

\textbf{Experiments on Synthetic Datasets.}
We first show that over-parametrization results in a less
non-convex loss surface for a synthetic dataset. The dataset consists of $1,000$ examples in $\R^{10}$ whose labels are generated by an underlying one-hidden-layer ReLU network $f(\x)=\sum_{i=1}^{I} \mathbf{w}_{i,2}^*[\mathbf{W}_{i,1}^*\mathbf{x}]_{+}$ with 11 hidden neurons~\cite{safran2017spurious}.
We make use of the visualization technique employed by \cite{li2017visualizing} to plot the landscape, where we project the high-dimensional hinge loss ($\tau=1$) landscape onto a 2-d plane spanned by three points. These points are found by running the SGD algorithm with three different initializations until the algorithm converges.
As shown in Figure~\ref{figure: experiment A}, the landscape exhibits strong non-convexity with lots of local minima in the under-parameterized case $I=10$. But as $I$ increases, the landscape becomes more convex. In the extreme case, when there are $1,000$ hidden neurons in the network, no non-convexity can be observed on the landscape.

\noindent{\textbf{Experiments on MNIST and CIFAR-10.}} We next verify the phenomenon of over-parametrization on MNIST~\cite{lecun1998gradient} and CIFAR-10~\cite{krizhevsky2009learning} datasets.
For both datasets, we follow the standard preprocessing step that each pixel is normalized by subtracting its mean and dividing by its standard deviation. We do not apply data augmentation. For MNIST,
we consider a single-hidden-layer network defined as: $f(\x)=\sum_{i=1}^{I} \mathbf{W}_{i,2}[\mathbf{W}_{i,1}\mathbf{x}]_{+}$, where $\mathbf{W}_{i,1}\in\mathbb{R}^{h\times d}$, $\mathbf{W}_{i,2}\in\mathbb{R}^{10\times h}$, $d$ is the input dimension, $h$ is the number of hidden neurons, and $I$ is the number of branches, with $d=784$ and $h=8$.
For CIFAR-10, in addition to considering the exact same one-hidden-layer architecture,
we also test a deeper network containing $3$ hidden layers of size $8$-$8$-$8$,
with ReLU activations and $d=3,072$.
We apply 10-class hinge loss on the top of the output of considered networks.

Figure~\ref{figure: experiment B} shows the changes of landscapes when $I$ increases from $1$ to $100$ for MNIST, and from $1$ to $50,000$ for CIFAR-10, respectively. When there is only one branch, the landscapes have strong non-convexity with many local minima. As the number of branches $I$ increases, the landscape becomes more convex. When $I=100$ for $1$-hidden-layer networks on MNIST and CIFAR-10, and $I=50,000$ for $3$-hidden-layer network on CIFAR-10, the landscape is almost convex.

\comment{
\begin{figure}
  \begin{subfigure}{0.24\textwidth}
    \includegraphics[width=\textwidth]{figures/C/01.pdf}
    \caption{$I=1$.}
  \end{subfigure}
  %
  %
  \begin{subfigure}{0.24\textwidth}
    \includegraphics[width=\textwidth]{figures/C/03.pdf}
    \caption{$I=3$.}
  \end{subfigure}
  %
  \begin{subfigure}{0.24\textwidth}
    \includegraphics[width=\textwidth]{figures/C/05.pdf}
    \caption{$I=5$.}
  \end{subfigure}
  %
  %
  %
  \begin{subfigure}{0.24\textwidth}
    \includegraphics[width=\textwidth]{figures/C/100.pdf}
    \caption{$I=100$.}
  \end{subfigure}

  \caption{Landscape of paralleled neural network on CIFAR-10.}
  \label{figure: experiment C}
\end{figure}
}

\comment{
\begin{figure}[h]
  \begin{center}
    \includegraphics[width=0.5\textwidth]{figures/D/vgg_9.pdf}
  \end{center}
  \caption{Performance of VGG-9 on CIFAR-10 with respect to number of paralleled paths.}
  \label{figure: experiment D}
\end{figure}
}

\subsection{Frequency of Hitting Global Minimum}
To further analyze the non-convexity of loss surfaces,
we consider various one-hidden-layer networks, where each network was
trained 100 times using different initialization seeds under the
setting discussed in our synthetic experiments of Section~\ref{section: Visualization of Loss Landscape}.
Since we have the ground-truth global minimum, we record the frequency that SGD hits the global minimum up to a small error $1\times10^{-5}$ after $100,000$ iterations. Table~\ref{table: likelihood reaching global minima}
shows that increasing the number of hidden neurons results in higher hitting rate of global optimality.
This further verifies that the loss surface of one-hidden-layer neural network becomes less non-convex as the width increases.

\begin{table}[h]
  \caption{\label{table: likelihood reaching global minima} Frequency of hitting global minimum by SGD with 100 different initialization seeds.}
  \begin{center}
  \begin{tabular}{c|c||c|c}
  \hline
  \# Hidden Neurons & Hitting Rate & \# Hidden Neurons & Hitting Rate\tabularnewline
  \hline
  \hline
  10 & 2 / 100 & 16 & 30 / 100\tabularnewline
  \hline
  11 & 9 / 100 & 17 & 32 / 100\tabularnewline
  \hline
  12 & 21 / 100 & 18 & 35 / 100\tabularnewline
  \hline
  13 & 24 / 100 & 19 & 52 / 100\tabularnewline
  \hline
  14 & 24 / 100 & 20 & 64 / 100\tabularnewline
  \hline
  15 & 29 / 100 & 21 & 75 / 100\tabularnewline
  \hline
  \end{tabular}
  \end{center}
\end{table}

\section{Conclusions}

In this work, we study the duality gap for two classes of network architectures. For the neural network with \emph{arbitrary activation functions}, multi-branch architecture and $\tau$-hinge loss, we show that the duality gap of both population and empirical risks shrinks to zero as the number of branches increases. Our result sheds light on better understanding the power of over-parametrization and the state-of-the-art architectures, where increasing the number of branches tends to make the loss surface less non-convex. For the neural network with linear activation function and $\ell_2$ loss, we show that the duality gap is zero. Our two results work for \emph{arbitrary depths} and \emph{adversarial data}, while the analytical techniques might be of independent interest to non-convex optimization more broadly.

\noindent \textbf{Acknowledgements.} We would like to thank Jason D. Lee for informing us the Shapley-Folkman lemma, and Maria-Florina Balcan, David P. Woodruff, Xiaofei Shi, and Xingyu Xie for their thoughtful comments on the paper.

\bibliographystyle{abbrv}
\bibliography{reference}

\newpage
\appendix
\section{Supplementary Experiments}
\subsection{Performance of Multi-Branch Architecture}

\begin{wrapfigure}{R}{7cm}
\vspace{-1.9cm}
\includegraphics[width=7cm]{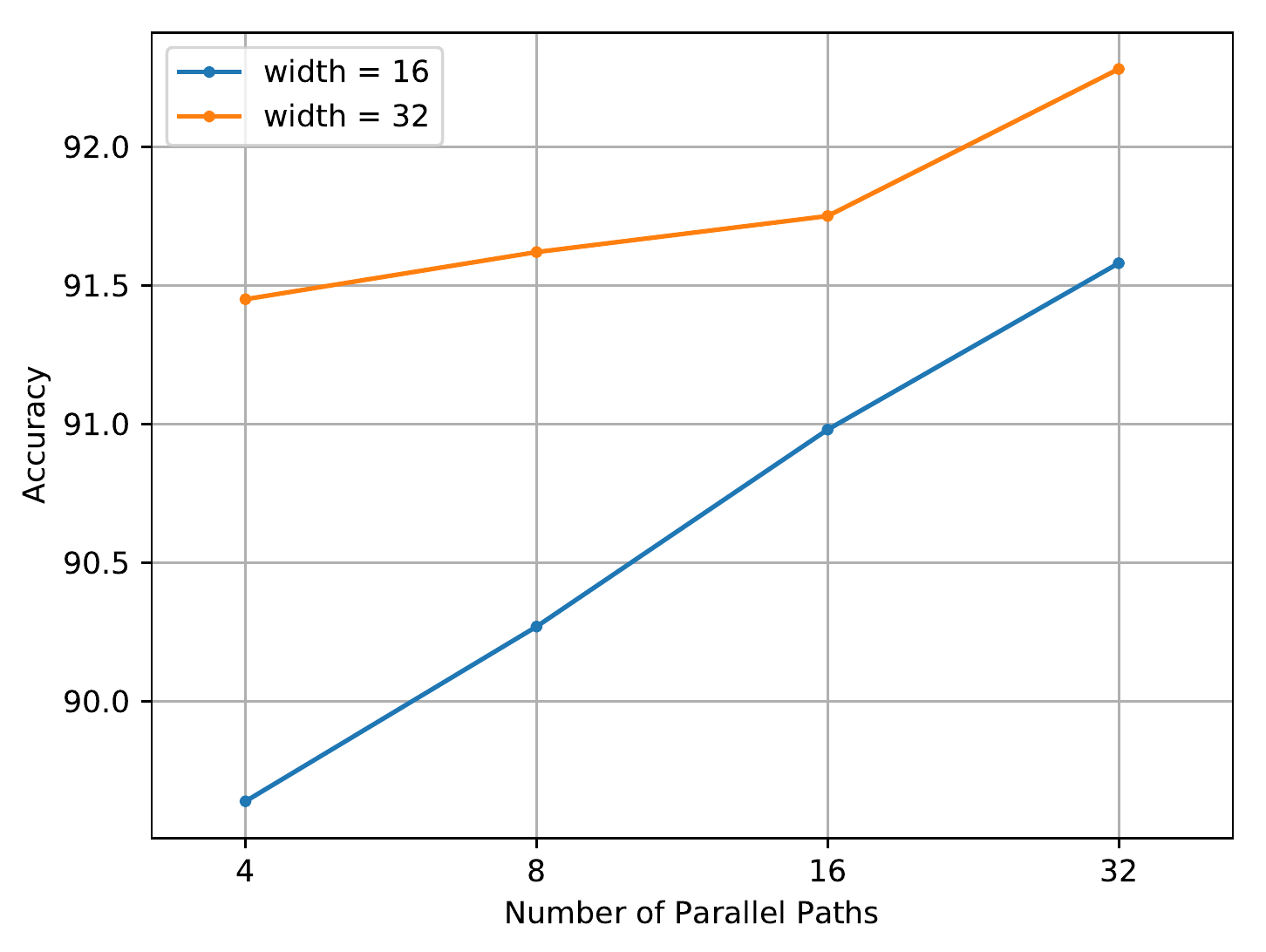}
\caption{\small{Test accuracy of using VGG-9 as the sub-networks in the multi-branch architecture.}}
\label{figure: experiment D}
\vspace{-0.5cm}
\end{wrapfigure}

\comment{
In Figures~\ref{figure: experiment A} and \ref{figure: experiment B}, the landscape becomes flatter when the number of paralleled paths increases, which probably indicates that the paralleled architecture empirically encourages SGD to find flat minima. As pointed out by many recent works~\cite{keskar2016large}, flat local minima typically generalize better. Hence empirically, the blessing of parallel architecture lies not only in convexity, but also generalizability in practice.
}

In this section, we test the classification accuracy of the multi-branch architecture on the CIFAR-10 dataset. We use a 9-layer VGG network~\cite{simonyan2014very} as our sub-network in each branch, which is memory-efficient for practitioners to fit many branches into GPU memory simultaneously. The detailed network setup of VGG-9 is in Table \ref{table: vgg-9 architecture}, where the width of VGG-9 is either 16 or 32. We test the performance of varying numbers of branches in the overall architecture from 4 to 32, with cross-entropy loss. Figure~\ref{figure: experiment D} presents the test accuracy on CIFAR-10 as the number of branches increases. It shows that the test accuracy improves monotonously with the increasing number of parallel branches/paths.

\begin{table}[h]
\caption{Network architecture of VGG-9. Here $w$ is the width of the network, which controls the number of filters in each convolution layer. All convolution layers have a kernel of size 3, and zero padding of size 1. All layers followed by the batch normalization have no bias term. All max pooling layers have a stride of 2.}
\vspace{-0.3cm}
\begin{center}
\begin{tabular}{c||c|c|c|c}
\hline
\textbf{Layer} & \textbf{Weight}            & \textbf{Activation} & \textbf{Input size}  & \textbf{Output size} \tabularnewline \hline\hline
Input          & N / A                      & N / A               & N / A                & $3\times32\times32$  \tabularnewline \hline
Conv1          & $3\times3\times3\times w$  & BN + ReLU           & $3\times32\times32$  & $w\times32\times32$  \tabularnewline \hline
Conv2          & $3\times3\times w\times w$ & BN + ReLU           & $w\times32\times32$  & $w\times32\times32$  \tabularnewline \hline
MaxPool        & N / A                      & N / A               & $w\times32\times32$  & $w\times16\times16$  \tabularnewline \hline
Conv3          & $3\times3\times w\times2w$ & BN + ReLU           & $w\times16\times16$  & $2w\times16\times16$ \tabularnewline \hline
Conv4          & $3\times3\times2w\times2w$ & BN + ReLU           & $2w\times16\times16$ & $2w\times16\times16$ \tabularnewline \hline
MaxPool        & N / A                      & N / A               & $2w\times16\times16$ & $2w\times8\times8$   \tabularnewline \hline
Conv5          & $3\times3\times2w\times4w$ & BN + ReLU           & $2w\times8\times8$   & $4w\times8\times8$   \tabularnewline \hline
Conv6          & $3\times3\times4w\times4w$ & BN + ReLU           & $4w\times8\times8$   & $4w\times8\times8$   \tabularnewline \hline
Conv7          & $3\times3\times4w\times4w$ & BN + ReLU           & $4w\times8\times8$   & $4w\times8\times8$   \tabularnewline \hline
MaxPool        & N / A                      & N / A               & $4w\times8\times8$   & $4w\times4\times4$   \tabularnewline \hline
Flatten        & N / A                      & N / A               & $4w\times4\times4$   & $64w$                \tabularnewline \hline
FC1            & $64w\times4w$              & BN + ReLU           & $64w$                & $4w$                 \tabularnewline \hline
FC2            & $4w\times10$               & Softmax             & $4w$                 & $10$                 \tabularnewline \hline
\end{tabular}
\end{center}
\label{table: vgg-9 architecture}
\vspace{-0.3cm}
\end{table}

\subsection{Strong Duality of Deep Linear Neural Networks \label{subsection: deep linear neural networks}}
We compare the optima of primal problem \eqref{equ: regularized PCA} and dual problem \eqref{equ: dual problem} by numerical experiments for three-layer linear neural networks ($H=3$). The data are generated as follows. We construct the output matrix $\Y\in\R^{100\times 100}$ by drawing the entries of $\Y$ from i.i.d. standard Gaussian distribution and the input matrix $\X\in\R^{100\times 100}$ by the identity matrix. The $d_{\min}$ varies from $5$ to $50$. Both primal and dual problems are solved by numerical algorithms. Given the non-convex nature of primal problem, we rerun the algorithm by multiple initializations and choose the best solution that we obtain. The results are shown in Figure \ref{figure: experiments}. We can easily see that the optima of primal and dual problems almost match. The small gap is due to the numerical inaccuracy.

\begin{figure}[h]
\centering
\includegraphics[width=0.5\textwidth]{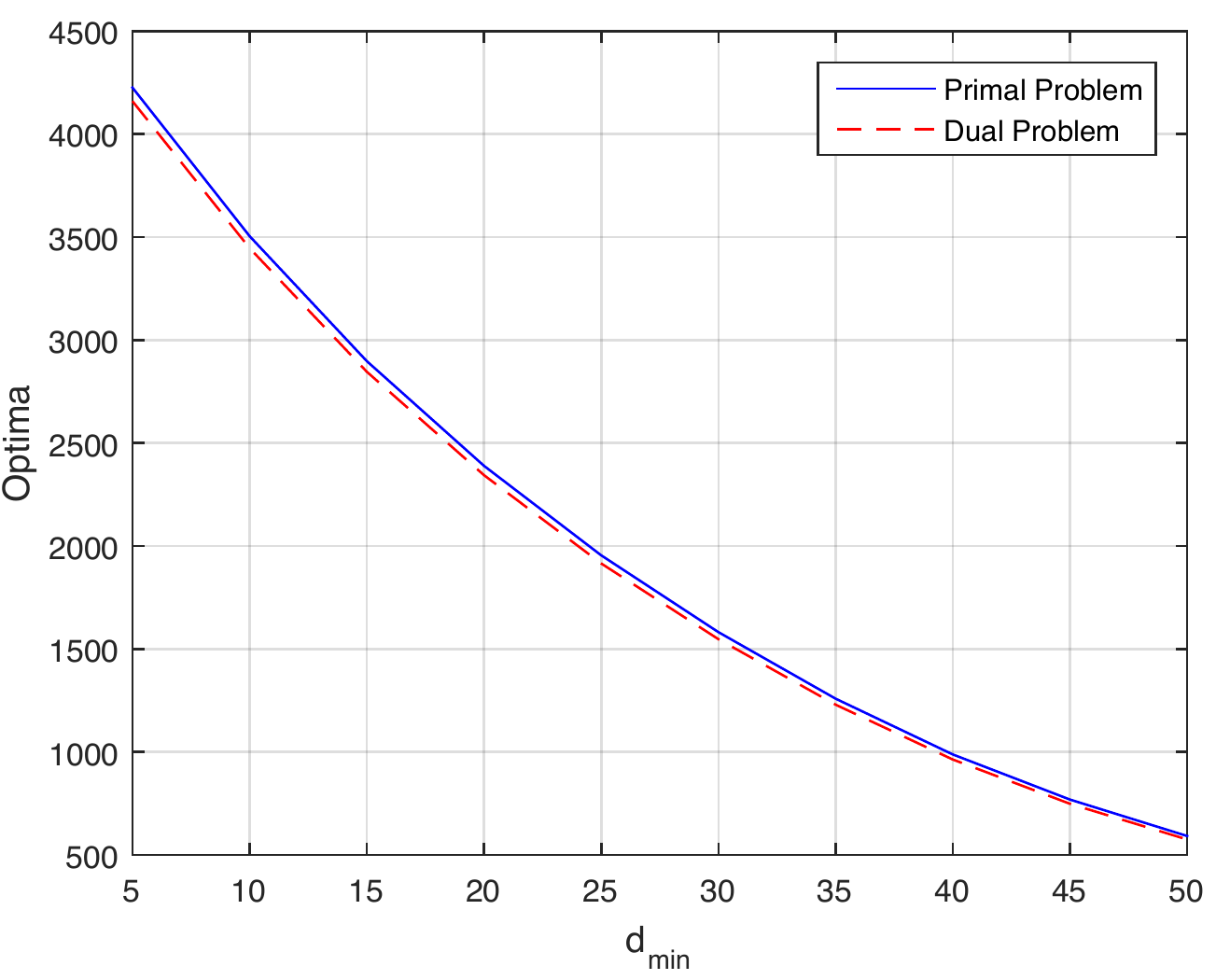}
\caption{Comparison of optima between primal and dual problems.}
\label{figure: experiments}
\end{figure}

We also compare the $\ell_2$ distance between the solution $\W_H^*\W_{H-1}^*\ldots\W_1^*$ of primal problem and the solution $\svd_{d_{\min}}(\tY-\bLambda^*)$ of dual problem in Table \ref{table: comparison of distance}. We see that the solutions are close to each other.
\begin{table}[h]
\caption{Comparison of the $\ell_2$ distance between the solutions of primal and dual problems.}
\label{table: comparison of distance}
\centering
\begin{tabular}{c|cccccccccc}%
\hline
$d_{\min}$ & 5 & 10 & 15 & 20 & 25 & 30 & 35 & 40 & 45 & 50\\\hline
$\ell_2$ distance ($\times 10^{-10}$) & $1.95$ & $1.26$ & $7.89$ & $3.80$ & 3.14 & 1.92 & 1.04 & 3.92 & 6.53 & 8.00\\
\hline
\end{tabular}
\vspace{-0.6cm}
\end{table}

\section{Proofs of Theorem \ref{theorem: deep non-linear neural networks}: Duality Gap of Multi-Branch Neural Networks}


The lower bound $0\le \frac{\inf(\mathbf{P})-\sup(\mathbf{D})}{\Delta_{worst}}$ is obvious by the weak duality. So we only need to prove the upper bound $\frac{\inf(\mathbf{P})-\sup(\mathbf{D})}{\Delta_{worst}}\le \frac{2}{I}$.

Consider the subset of $\mathbb{R}^2$:
\begin{equation*}
\mathcal{Y}_i:=\left\{\y_i\in\R^2:\y_i=\frac{1}{I}\left[h_i(\w_{(i)}),\bbE_{(\x,y)\sim \cP}\left(1-\frac{y\cdot f_i(\w_{(i)};\x)}{\tau}\right)\right],\w_{(i)}\in\cW_i\right\},\quad i\in[I].
\end{equation*}
Define the vector summation
\begin{equation*}
\cY:=\cY_1+\cY_2+...+\cY_I.
\end{equation*}
Since $f_i(\w_{(i)};\x)$ and $h_i(\w_{(i)})$ are continuous w.r.t. $\w_{(i)}$ and $\cW_i$'s are compact, the set
\begin{equation*}
\{(\w_{(i)},h_i(\w_{(i)}),f_i(\w_{(i)};\x)):\w_{(i)}\in\cW_i\}
\end{equation*}
is compact as well. So $\cY$, $\mathsf{conv}(\cY)$, $\cY_i$, and $\mathsf{conv}(\cY_i)$, $i\in[I]$ are all compact sets.
According to the definition of $\cY$ and the standard duality argument~\cite{magnanti1976generalized}, we have
\begin{equation*}
\inf(\mathbf{P})=\min\left\{w:\mbox{there exists }(r,w)\in \cY\mbox{ such that }r\le K\right\},
\end{equation*}
and
\begin{equation*}
\sup(\mathbf{D})=\min\left\{w:\mbox{there exists }(r,w)\in\mathsf{conv}\left(\cY\right)\mbox{ such that }r\le K\right\}.
\end{equation*}

\medskip
\noindent{\textbf{Technique (a): Shapley-Folkman Lemma.}} We are going to apply the following Shapley-Folkman lemma.
\begin{lemma}[Shapley-Folkman, \cite{starr1969quasi}]
\label{lemma: Shapley-Folkman}
Let $\cY_i,i\in[I]$ be a collection of subsets of $\mathbb{R}^m$. Then for every $\y\in\mathsf{conv}(\sum_{i=1}^I \cY_i)$, there is a subset $\cI(\y)\subseteq[I]$ of size at most $m$ such that
\begin{equation*}
\y\in\left[\sum_{i\not\in \cI(\y)} \cY_i+\sum_{i\in \cI(\y)}\mathsf{conv}(\cY_i)\right].
\end{equation*}
\end{lemma}

We apply Lemma \ref{lemma: Shapley-Folkman} to prove Theorem \ref{theorem: deep non-linear neural networks} with $m=2$. Let $(\overline{r},\overline{w})\in\mathsf{conv}(\cY)$ be such that
\begin{equation*}
\quad \overline{r}\le K,\quad \mbox{and}\quad \overline{w}=\sup(\mathbf{D}).
\end{equation*}
Applying the above Shapley-Folkman lemma to the set $\cY=\sum_{i=1}^I\cY_i$, we have that there are a subset $\overline{\cI}\subseteq[I]$ of size $2$ and vectors
\begin{equation*}
(\overline{r}_i,\overline{w}_i)\in\mathsf{conv}(\cY_i),\ \ i\in\overline{\cI}\qquad\mbox{and}\qquad \overline{\w}_{(i)}\in\cW_i,\ \ i\not\in\overline{\cI},
\end{equation*}
such that
\begin{equation}
\label{equ: step 1}
\frac{1}{I}\sum_{i\not\in\overline{\cI}} h_i(\overline{\w}_{(i)})+\sum_{i\in\overline{\cI}} \overline{r}_i=\overline{r}\le K,
\end{equation}
\begin{equation}
\label{equ: step 2}
\frac{1}{I}\sum_{i\not\in\overline{\cI}} \bbE_{(\x,y)\sim \cP}\left(1-\frac{y\cdot f_i(\overline\w_{(i)};\x)}{\tau}\right)+\sum_{i\in\overline{\cI}}\overline{w}_i=\sup(\mathbf{D}).
\end{equation}
Representing elements of the convex hull of $\cY_i\subseteq\R^2$ by Carath\'{e}odory theorem, we have that for each $i\in\overline{\cI}$, there are vectors $\w_{(i)}^1,\w_{(i)}^2,\w_{(i)}^3\in\cW_i$ and scalars $a_i^1,a_i^2,a_i^3\in\mathbb{R}$ such that
\begin{equation*}
\sum_{j=1}^3 a_i^j=1,\quad a_i^j\ge 0,\ j=1,2,3,
\end{equation*}
\begin{equation*}
\overline{r}_i=\frac{1}{I}\sum_{j=1}^3 a_i^j h_i(\w_{(i)}^j),\qquad \overline{w}_i=\frac{1}{I}\sum_{j=1}^3 a_i^j \bbE_{(\x,y)\sim \cP}\left(1-\frac{y\cdot f_i(\w_{(i)}^j;\x)}{\tau}\right).
\end{equation*}
Recall that we define
\begin{equation}
\label{equ: definition of widehat f}
\widehat{f}_i(\widetilde{\w}):=\inf_{\w_{(i)}\in\cW_i}\left\{\bbE_{(\x,y)\sim \cP} \left(1-\frac{y\cdot f_i(\w_{(i)};\x)}{\tau}\right):h_i(\w_{(i)})\le h_i(\widetilde{\w})\right\},
\end{equation}
\begin{equation*}
\widetilde{f}_i(\widetilde\w):=\inf_{a^j,\w_{(i)}^j\in\cW_i}\left\{\sum_{j=1}^{p_i+1}a^j \bbE_{(\x,y)\sim \cP} \left(1-\frac{y\cdot f_i(\w_{(i)}^j;\x)}{\tau}\right):\widetilde{\w}=\sum_{j=1}^{p_i+1}a^j\w_{(i)}^j,\sum_{j=1}^{p_i+1}a^j=1,a^j\ge 0\right\},
\end{equation*}
and $\Delta_i:=\sup_{\w\in\cW_i}\left\{\widehat{f}_i(\w)-\widetilde{f}_i(\w)\right\}\ge 0$. We have for $i\in\overline \cI$,
\begin{equation}
\label{equ: step 3}
\overline r_i\ge \frac{1}{I}h_i\left(\sum_{j=1}^3 a_i^j\w_{(i)}^j\right),\quad (\text{because $h_i(\cdot)$ is convex})
\end{equation}
and
\begin{equation}
\label{equ: step 4}
\begin{split}
\overline{w}_i&\ge \frac{1}{I}\widetilde{f}_i\left(\sum_{j=1}^3 a_i^j\w_{(i)}^j\right)\quad\text{(by the definition of $\widetilde f_i(\cdot)$)}\\
&\ge \frac{1}{I}\widehat f_i\left(\sum_{j=1}^3 a_i^j\w_{(i)}^j\right)-\frac{1}{I}\Delta_i.\quad\text{(by the definition of $\Delta_i$)}
\end{split}
\end{equation}
Thus, by Eqns. \eqref{equ: step 1} and \eqref{equ: step 3}, we have
\begin{equation}
\label{equ: step 5}
\frac{1}{I}\sum_{i\not\in\overline{\cI}} h_i(\overline{\w}_{(i)})+\frac{1}{I}\sum_{i\in\overline{\cI}} h_i\left(\sum_{j=1}^3 a_i^j\w_{(i)}^j\right)\le K,
\end{equation}
and by Eqns. \eqref{equ: step 2} and \eqref{equ: step 4}, we have
\begin{equation}
\label{equ: step 6}
\bbE_{(\x,y)\sim \cP}\left[ \frac{1}{I} \sum_{i\not\in\overline{\cI}}  \left(1-\frac{y\cdot f_i(\overline\w_{(i)};\x)}{\tau}\right)\right]+\frac{1}{I}\sum_{i\in\overline{\cI}}\widehat f_i\left(\sum_{j=1}^3 a_i^j\w_{(i)}^j\right)\le \sup(\mathbf{D})+\frac{1}{I}\sum_{i\in\overline \cI}\Delta_i.
\end{equation}
Given any $\epsilon>0$ and $i\in\overline \cI$, we can find a vector $\overline\w_{(i)}\in\cW_i$ such that
\begin{equation}
\label{equ: middle inequality}
h_i(\overline \w_{(i)})\le h_i\left(\sum_{j=1}^3 a_i^j\w_{(i)}^j\right)\ \text{and}\ \bbE_{(\x,y)\sim \cP} \left(1-\frac{y\cdot f_i(\overline\w_{(i)};\x)}{\tau}\right)\le \widehat f_i\left(\sum_{j=1}^3 a_i^j\w_{(i)}^j\right)+\epsilon,
\end{equation}
where the first inequality holds because $\cW_i$ is convex and the second inequality holds by the definition \eqref{equ: definition of widehat f} of $\widehat f_i(\cdot)$.
Therefore, Eqns. \eqref{equ: step 5} and \eqref{equ: middle inequality} impliy that
\begin{equation*}
\frac{1}{I}\sum_{i=1}^I h_i(\overline{\w}_{(i)})\le K.
\end{equation*}
Namely, $(\overline{\w}_{(1)},...,\overline{\w}_{(I)})$ is a feasible solution of problem \eqref{equ: problem (P)}. Also, Eqns. \eqref{equ: step 6} and \eqref{equ: middle inequality} yield
\begin{equation*}
\begin{split}
\inf(\mathbf{P})&\le \bbE_{(\x,y)\sim \cP} \left[\frac{1}{I}\sum_{i=1}^I \left(1-\frac{y\cdot f_i(\overline\w_{(i)};\x)}{\tau}\right)\right]\\
&\le \sup(\mathbf{D})+\frac{1}{I}\sum_{i\in\overline{\cI}}(\Delta_i+\epsilon)\\
&\le \sup(\mathbf{D})+\frac{2}{I}\Delta_{worst}+2\epsilon,
\end{split}
\end{equation*}
where the last inequality holds because $|\overline \cI|=2$.
Finally, letting $\epsilon\rightarrow 0$ leads to the desired result.

\section{Proofs of Theorem \ref{theorem: local=global}: Strong Duality of Deep Linear Neural Networks}
\label{section: Strong Duality of Deep Linear Neural Networks}

Let $\widetilde{\Y}=\Y\X^\dag\X$. We note that by Pythagorean theorem, for every $\Y$,
\begin{equation*}
\begin{split}
\frac{1}{2}\|\Y-\W_H\cdots\W_1\X\|_F^2=\frac{1}{2}\|\widetilde\Y-\W_H\cdots\W_1\X\|_F^2+\underbrace{\frac{1}{2}\|\Y-\widetilde\Y\|_F^2}_{\text{independent of }\W_1,\ldots,\W_H}.
\end{split}
\end{equation*}
So we can focus on the following optimization problem instead of problem \eqref{equ: regularized PCA}:
\begin{equation}
\label{equ: regularized PCA app}
\min_{\W_1,...,\W_H} \frac{1}{2}\|\widetilde\Y-\W_H...\W_1\X\|_F^2+\frac{\gamma}{H}\left[ \|\W_1\X\|_{\cS_H}^H + \sum_{i=2}^{H}\|\W_i\|_{\cS_H}^H \right].
\end{equation}

\medskip
\noindent{\textbf{Technique (b): Variational Form.}} Our work is inspired by a variational form of problem \eqref{equ: regularized PCA app} given by the following lemma.

\medskip
\begin{lemma}
\label{lemma: variational form of nuclear norm}
If $(\W_1^{*},\ldots,\W_H^{*})$ is optimal to problem
\begin{equation}
\label{equ: nuclear norm form}
\min_{\W_1,\ldots,\W_H} F(\W_1,\ldots,\W_H) := \frac{1}{2}\|\tY-\W_H\cdots\W_1\X\|_F^2+\gamma\|\W_H\cdots\W_1\X\|_*,
\end{equation}
then $(\W_1^{**},\ldots,\W_H^{**})$ is optimal to problem \eqref{equ: regularized PCA app}, where $\U\bSigma\V^T$ is the skinny SVD of $\W_H^*\W_{H-1}^*\cdots\W_1^*\X$, $\W_i^{**}=[\bSigma^{1/H},\0;\0,\0]\in\R^{d_i\times d_{i-1}}$ for $i=2,3,...,H-1$, $\W_H^{**}=[\U\bSigma^{1/H},\0]\in\R^{d_H\times d_{H-2}}$ and $\W_1^{**}=[\bSigma^{1/H}\V^T;\0]\X^\dag\in\R^{d_1\times d_0}$. Furthermore, problems \eqref{equ: regularized PCA app} and \eqref{equ: nuclear norm form} have the same optimal objective function value.
\comment{
Moreover, the optimum is achieved when for $i=1,2,\ldots,d_{\min}$,
$$
\sigma_i^H(\W_1^*\X)=\sigma_i^H(\W_2^*)=\cdots =\sigma_i^H(\W_H^*)=\sigma_i(\W_H^*\W_{H-1}^*\cdots\W_1^*\X) .
$$
}
\end{lemma}

\begin{proof}[Proof of Lemma \ref{lemma: variational form of nuclear norm}]
Let $\U\bSigma\V^T$ be the skinny SVD of matrix $\W_H\W_{H-1}\cdots\W_1\X=:\Z$. We notice that
\begin{equation*}
\begin{split}
\|\Z\|_* & = \|\W_H\W_{H-1}\cdots\W_1\X\|_*\\
& \leq \|\W_1\X\|_{\cS_H}\prod_{i=2}^H \|\W_i\|_{\cS_H} \quad \text{(by the generalized H$\ddot{\textup{o}}$lder's inequality)}
\\
& \leq \frac{1}{H} \left[ \|\W_1\X\|_{\cS_H}^H + \sum_{i=2}^{H}\|\W_i\|_{\cS_H}^H \right].\quad(\text{by the inequality of mean} )
\end{split}
\end{equation*}
Hence, on one hand, for every $(\W_1, \ldots, \W_H)$,
\begin{equation}
\begin{split}
\min_{\W_1,\ldots,\W_H} F(\W_1,\ldots,\W_H)
&\leq \frac{1}{2}\|\tY-\W_H\cdots\W_1\X\|_F^2+\gamma\|\W_H\W_{H-1}\cdots\W_1\X\|_* \\
&\leq \frac{1}{2}\|\tY-\W_H\cdots\W_1\X\|_F^2+\frac{\gamma}{H} \left[ \|\W_1\X\|_{\cS_H}^H + \sum_{i=2}^{H}\|\W_i\|_{\cS_H}^H \right],
\nonumber
\end{split}
\end{equation}
which yields
\begin{equation*}
\min_{\W_1,\ldots,\W_H} F(\W_1,\ldots,\W_H)
\leq
\min_{\W_1,\ldots,\W_H} \frac{1}{2}\|\tY-\W_H\cdots\W_1\X\|_F^2+\frac{\gamma}{H} \left[ \|\W_1\X\|_{\cS_H}^H + \sum_{i=2}^{H}\|\W_i\|_{\cS_H}^H \right].
\end{equation*}
On the other hand, suppose $(\W_1^*,\ldots,\W_H^*)$ is optimal to problem \eqref{equ: nuclear norm form}, and let $\U\bSigma\V^T$ be the skinny SVD of matrix $\W^*_H\W^*_{H-1}\cdots\W^*_1\X$. We choose $(\W_1^{**},\ldots,\W_H^{**})$ such that
$$
\W^{**}_H = [\U\bSigma^{\frac{1}{H}}, \0], \;\W^{**}_1\X = [\bSigma^{\frac{1}{H}}\V^T; \0], \; \W^{**}_i = [\bSigma^{\frac{1}{H}}, \0; \0, \0], \; i=2, \ldots, H-1.
$$
We pad $\0$ around $\W_i^{**}$ so as to adapt to the dimensionality of each $\W_i^{**}$.
Notice that
\begin{equation}
\begin{split}
\|\W^*_H\W^*_{H-1}\cdots\W^*_1\X\|_*
&= \|\W^{**}_H\W^{**}_{H-1}\cdots\W^{**}_1\X\|_* \\
&= \frac{1}{H} \left[ \|\W^{**}_1\X\|_{\cS_H}^H + \sum_{i=2}^{H}\|\W^{**}_i\|_{\cS_H}^H \right].
\nonumber
\end{split}
\end{equation}
Since $\W^*_H\W^*_{H-1}\cdots\W^*_1\X=\W^{**}_H\W^{**}_{H-1}\cdots\W^{**}_1\X$, for every $\tY$,
$$
\|\tY - \W^*_H\W^*_{H-1}\cdots\W^*_1\X\|_F = \|\tY - \W^{**}_H\W^{**}_{H-1}\cdots\W^{**}_1\X\|_F.
$$
Hence
\begin{equation*}
\begin{split}
& \min_{\W_1, \ldots, \W_H} F(\W_1,\ldots,\W_H) = F(\W^*_1,\ldots, \W^*_H) = F(\W^{**}_1,\ldots, \W^{**}_H)\\
& = \frac{1}{2}\|\tY-\W^{**}_H\cdots\W^{**}_1\X\|_F^2+\frac{\gamma}{H} \left[ \|\W^{**}_1\X\|_{\cS_H}^H + \sum_{i=2}^{H}\|\W^{**}_i\|_{\cS_H}^H \right] \\
& \geq \min_{\W_1,\ldots,\W_H} \frac{1}{2}\|\tY-\W_H\cdots\W_1\X\|_F^2+\frac{\gamma}{H} \left[ \|\W_1\X\|_{\cS_H}^H + \sum_{i=2}^{H}\|\W_i\|_{\cS_H}^H \right],
\end{split}
\end{equation*}
which yields the other direction of the inequality and hence completes the proof.
\end{proof}

\medskip
\noindent{\textbf{Technique (c): Reduction to Low-Rank Approximation.}} We now reduce problem \eqref{equ: nuclear norm form} to the classic problem of low-rank approximation of the form $\min_{\W_1,\ldots,\W_H} \frac{1}{2}\|\widehat{\Y}-\W_H\cdots\W_1\X\|_F^2$, which has the following nice properties.

\medskip
\begin{lemma}
\label{lemma: local global}
For any $\widehat \Y\in\mathsf{Row}(\X)$, every global minimum $(\W_1^*,\ldots,\W_H^*)$ of function $$f(\W_1,\ldots,\W_H)=\frac{1}{2}\|\widehat \Y-\W_H\cdots\W_1\X\|_F^2$$ obeys $\W_H^*\cdots\W_1^*\X=\svd_{d_{\min}}({\widehat\Y})$. Here $\widehat \Y\in\mathsf{Row}(\X)$ means the row vectors of $\widehat \Y$ belongs to the row space of $\X$.
\end{lemma}

\begin{proof}[Proof of Lemma \ref{lemma: local global}]
Note that the optimal solution to $\min_{\W_H,\ldots,\W_1} \frac{1}{2}\|\widehat{\Y}-\W_H\cdots\W_1\X\|_F^2$ is equal to the optimal solution to the low-rank approximation problem $\min_{\rank(\Z)\le d_{\min}} \frac{1}{2}\|\widehat{\Y}-\Z\|_F^2$ when $\widehat \Y\in\mathsf{Row}(\X)$, which has a closed-form solution $\svd_{d_{\min}}({\widehat\Y})$.\footnote{Note that the low-rank approximation problem might have non-unique solution. However, we will use in this paper the abuse of language $\svd_{d_{\min}}({\widehat\Y})$ as the non-uniqueness issue does not lead to any issue in our developments.}
\end{proof}

We now reduce $F(\W_1,\ldots,\W_H)$ to the form of $\frac{1}{2}\|\widehat{\Y}-\W_H\cdots\W_1\X\|_F^2$ for some $\widehat \Y$ plus an extra additive term that is independent of $(\W_1,\ldots,\W_H)$. To see this, denote by $K(\cdot)=\gamma\|\cdot\|_*$. We have
\begin{equation*}
\begin{split}
\quad F(\W_1,\ldots,\W_H)&=\frac{1}{2}\|\tY-\W_H\cdots\W_1\X\|_F^2+K^{**}(\W_H\cdots\W_1\X)\\
&=\max_\bLambda\frac{1}{2}\|\tY-\W_H\cdots\W_1\X\|_F^2+\langle\bLambda,\W_H\cdots\W_1\X\rangle-K^*(\bLambda)\\
&=\max_\bLambda \frac{1}{2}\|\tY-\bLambda-\W_H\cdots\W_1\X\|_F^2-\frac{1}{2}\|\bLambda\|_F^2-K^*(\bLambda)+\langle\tY,\bLambda\rangle\\
&=: \max_\bLambda L(\W_1,\ldots,\W_H,\bLambda),
\end{split}
\end{equation*}
where we define $L(\W_1,\ldots,\W_H,\bLambda):=\frac{1}{2}\|\tY-\bLambda-\W_H\cdots\W_1\X\|_F^2-\frac{1}{2}\|\bLambda\|_F^2-K^*(\bLambda)+\langle\tY,\bLambda\rangle$ as the Lagrangian of problem \eqref{equ: nuclear norm form}. The first equality holds because $K(\cdot)$ is closed and convex w.r.t. the argument $\W_H\cdots\W_1\X$ so $K(\cdot)=K^{**}(\cdot)$, and the second equality is by the definition of conjugate function. One can check that $L(\W_1,\ldots,\W_H,\bLambda)=\min_\M L'(\W_1,\ldots,\W_H,\M,\bLambda)$, where $L'(\W_1,\ldots,\W_H,\M,\bLambda)$ is the Lagrangian of the constraint optimization problem $\min_{\W_1,\ldots,\W_H,\M} \frac{1}{2}\|\tY-\W_H\cdots\W_1\X\|_F^2+K(\M),\ \mbox{s.t.}\ \M=\W_H\cdots\W_1\X$. With a little abuse of notation, we call $L(\A,\B,\bLambda)$ the Lagrangian of the unconstrained problem \eqref{equ: nuclear norm form} as well.

The remaining analysis is to choose a proper $\bLambda^*\in\mathsf{Row}(\X)$ such that $(\W_1^*,\ldots,\W_H^*,\bLambda^*)$ is a primal-dual saddle point of $L(\W_1,\ldots,\W_H,\bLambda)$, so that the problem $\min_{\W_1,\ldots,\W_H} L(\W_1,\ldots,\W_H,\bLambda^*)$ and problem \eqref{equ: nuclear norm form} have the same optimal solution $(\W_1^*,\ldots,\W_H^*)$.
For this, we introduce the following condition, and later we will show that the condition holds.

\medskip
\begin{condition}
\label{lemma: Lagrangian multiplier}
For a solution $(\W_1^*,\ldots,\W_H^*)$ to optimization problem \eqref{equ: nuclear norm form}, there exists an $$\bLambda^*\in\partial_{\Z} K(\Z)|_{\Z=\W_H^*\cdots\W_1^*\X}\cap\mathsf{Row}(\X)$$ such that
\begin{equation}
\begin{split}
\label{equ: Lambda condition}
\W_{i+1}^{*T}\cdots\W_H^{*T}(\W_H^*\cdots\W_1^*\X+\bLambda^*-\widetilde\Y)\X^T\W_1^{*T}\cdots\W_{i-1}^{*T}&=\0,\quad i=2,\ldots,H-1,\\
 \W_2^{*T}\cdots\W_H^{*T}(\W_H^*\W_{H-1}^*\cdots\W_1^*\X+\bLambda^*-\widetilde\Y)\X^T&=\0,\\
 (\W_H^*\W_{H-1}^*\cdots\W_1^*\X+\bLambda^*-\widetilde\Y)\X^T\W_1^{*T}\cdots\W_{H-1}^{*T}&=\0.\\
\end{split}
\end{equation}
\end{condition}

\noindent
We note that if we set $\bLambda$ to be the $\bLambda^*$ in \eqref{equ: Lambda condition}, then $\nabla_{\W_i} L(\W_1^*,\ldots,\W_H^*,\bLambda^*)=\0$ for every $i$. So $(\W_1^*,\ldots,\W_H^*)$ is either a saddle point, a local minimizer, or a global minimizer of $L(\W_1,\ldots,\W_H,\bLambda^*)$ as a function of $(\W_1,\ldots,\W_H)$ for the fixed $\bLambda^*$. The following lemma states that if it is a global minimizer, then strong duality holds.

\medskip
\begin{lemma}
\label{lemma: f and k local minimum}
Let $(\W_1^*,\ldots,\W_H^*)$ be a global minimizer of $F(\W_1,\ldots,\W_H)$. If there exists a dual certificate $\bLambda^*$ satisfying Condition \ref{lemma: Lagrangian multiplier} and the pair $(\W_1^*,\ldots,\W_H^*)$ is a global minimizer of $L(\W_1,\ldots,\W_H,\bLambda^*)$ for the fixed $\bLambda^*$, then strong duality holds. Moreover, we have the relation $\W_H^*\cdots\W_1^*\X=\textup{\textsf{svd}}_{d_{\min}}(\tY-\bLambda^*)$.
\end{lemma}

\begin{proof}[Proof of Lemma \ref{lemma: f and k local minimum}]
By the assumption of the lemma, $(\W_1^*,\ldots,\W_H^*)$ is a global minimizer of $$L(\W_1,\ldots,\W_H,\bLambda^*)=\frac{1}{2}\|\tY-\bLambda^*-\W_H\W_{H-1}\cdots\W_1\X\|_F^2+c(\bLambda^*),$$ where $c(\bLambda^*)$ is a function of $\bLambda^*$ that is independent of $\W_i$ for all $i$'s.
Namely, $(\W_1^*,\ldots,\W_H^*)$ globally minimizes $L(\W_1,\ldots,\W_H,\bLambda)$ when $\bLambda$ is fixed to $\bLambda^*$. Furthermore, $\bLambda^*\in\partial_{\Z} K(\Z)|_{\Z=\W_H^*\ldots\W_1^*\X}$ implies that $\W_H^*\W_{H-1}^*\cdots\W_1^*\X\in\partial_\bLambda K^*(\bLambda)|_{\bLambda=\bLambda^*}$ by the convexity of function $K(\cdot)$, meaning that $\0\in \partial_\bLambda L(\W_1^*,\ldots,\W_H^*,\bLambda)$. So $\bLambda^*=\argmax_{\bLambda} L(\W_1^*,\ldots,\W_H^*,\bLambda)$ due to the concavity of function $L(\W_1^*,\ldots,\W_H^*,\bLambda)$ w.r.t. variable $\bLambda$. Thus $(\W_1^*,\ldots,\W_H^*,\bLambda^*)$ is a primal-dual saddle point of $L(\W_1,\ldots,\W_H,\bLambda)$.

We now prove the strong duality. By the fact that $F(\W_1^*,\ldots,\W_H^*)=\max_\bLambda L(\W_1^*,\ldots,\W_H^*,\bLambda)$ and that $\bLambda^*=\argmax_{\bLambda} L(\W_1^*,\ldots,\W_H^*,\bLambda)$, for every $\W_1,\ldots,\W_H$, we have
\begin{equation*}
F(\W_1^*,\ldots,\W_H^*)=L(\W_1^*,\ldots,\W_H^*,\bLambda^*)\le L(\W_1,\ldots,\W_H,\bLambda^*),
\end{equation*}
where the inequality holds because $(\W_1^*,\ldots,\W_H^*,\bLambda^*)$ is a primal-dual saddle point of $L$. Notice that we also have
\begin{equation*}
\begin{split}
\min_{\W_1,\ldots,\W_H}\max_\bLambda L(\W_1,\ldots,\W_H,\bLambda)&=F(\W_1^*,\ldots\W_H^*)\\&\le \min_{\W_1,\ldots,\W_H} L(\W_1,\ldots,\W_H,\bLambda^*)\\&\le\max_\bLambda\min_{\W_1,\ldots,\W_H} L(\W_1,\ldots,\W_H,\bLambda).
\end{split}
\end{equation*}
On the other hand, by weak duality,
\begin{equation*}
\min_{\W_1,\ldots,\W_H}\max_\bLambda L(\W_1,\ldots,\W_H,\bLambda)\ge \max_\bLambda\min_{\W_1,\ldots,\W_H} L(\W_1,\ldots,\W_H,\bLambda).
\end{equation*}
Therefore, $$\min_{\W_1,\ldots,\W_H}\max_\bLambda L(\W_1,\ldots,\W_H,\bLambda)=\max_\bLambda\min_{\W_1,\ldots,\W_H} L(\W_1,\ldots,\W_H,\bLambda),$$ i.e., strong duality holds. Hence,
\begin{equation*}
\begin{split}
\W_H^*&\W_{H-1}^*\cdots\W_1^*=\argmin_{\W_H\W_{H-1}\ldots\W_1} L(\W_1,\ldots,\W_H,\bLambda^*)\\
&=\argmin_{\W_H\W_{H-1}\cdots\W_1} \frac{1}{2}\|\tY-\bLambda^*-\W_H\W_{H-1}\cdots\W_1\X\|_F^2-\frac{1}{2}\|\bLambda^*\|_F^2-K^*(\bLambda^*)+\langle\tY,\bLambda^*\rangle\\
&=\argmin_{\W_H\W_{H-1}\cdots\W_1} \frac{1}{2}\|\tY-\bLambda^*-\W_H\W_{H-1}\cdots\W_1\X\|_F^2\\
&=\textsf{svd}_{d_{\min}}(\tY-\bLambda^*).
\end{split}
\end{equation*}
The proof of Lemma \ref{lemma: f and k local minimum} is completed.
\end{proof}

\noindent{\textbf{Technique (d): Dual Certificate.}} We now construct dual certificate $\bLambda^*$ such that all of conditions in Lemma \ref{lemma: f and k local minimum} hold.
We note that $\bLambda^*$ should satisfy the followings by Lemma \ref{lemma: f and k local minimum}:
\begin{flalign}
\label{equ: dual conditions for general problem}
\begin{split}
&\mbox{(a)} \quad \bLambda^*\in\partial K(\W_H^*\W_{H-1}^*\cdots\W_1^*\X)\cap\mathsf{Row}(\X);\qquad \mbox{(by Condition \ref{lemma: Lagrangian multiplier})}\\
&\mbox{(b)} \quad \text{Equations }\eqref{equ: Lambda condition};\qquad \mbox{(by Condition \ref{lemma: Lagrangian multiplier})}\\
&\mbox{(c)} \quad \W_H^*\W_{H-1}^*\cdots\W_1^*\X=\textsf{svd}_{d_{\min}}(\tY-\bLambda^*). \quad \mbox{(by the global optimality and Lemma \ref{lemma: local global})}
\end{split}
\end{flalign}
Before proceeding, we denote by $\widetilde\A:=\W_H^*\cdots\W_{\min+1}^*$, $\widetilde\B:=\W_{\min}^*\cdots\W_1^*\X$, where $\W_{\min}^*$ is a matrix among $\{\W_i^*\}_{i=1}^{H-1}$ which has $d_{\min}$ rows, and let
$$\cT := \{\widetilde\A\C_1^T+\C_2\widetilde\B:\ \C_1\in\R^{n\times d_{\min}},\ \C_2\in\R^{d_H\times d_{\min}}\}$$
be a matrix space. Denote by $\cU$ the left singular space of $\widetilde\A\widetilde\B$ and $\cV$ the right singular space. Then the linear space $\cT$ can be equivalently represented as $\cT=\cU+\cV$. Therefore, $\cT^\perp=(\cU+\cV)^\perp=\cU^\perp\cap\cV^\perp$. With this, we note that: (b) $\W_H^*\W_{H-1}^*\cdots\W_1^*\X+\bLambda^*-\tY\in\mathsf{Null}(\widetilde\A^T)=\mathsf{Col}(\widetilde\A)^\perp$ and $\W_H^*\W_{H-1}^*\cdots\W_1^*\X+\bLambda^*-\tY\in\mathsf{Row}(\widetilde\B)^\perp$ (so $\W_H^*\W_{H-1}^*\cdots\W_1^*\X+\bLambda^*-\tY\in\cT^\perp$) imply Equations \eqref{equ: Lambda condition} since either $\W_{i+1}^{*T}\cdots\W_H^{*T}(\W_H^*\W_{H-1}^*\cdots\W_1^*\X+\bLambda^*-\tY)=\0$ or $(\W_H^*\W_{H-1}^*\cdots\W_1^*\X+\bLambda^*-\tY)\X^T\W_1^{*T}\cdots\W_{i-1}^{*T}=\0$ for all $i$'s.
And (c) for an orthogonal decomposition
$\tY-\bLambda^*=\W_H^*\W_{H-1}^*\cdots\W_1^*\X+\E \mbox{ where } \W_H^*\W_{H-1}^*\cdots\W_1^*\X\in\cT \mbox{ and }\E\in\cT^\perp$, we have that $$\|\E\|\le \sigma_{d_{\min}}(\W_H^*\W_{H-1}^*\cdots\W_1^*\X)$$ and condition (b) together imply $\W_H^*\W_{H-1}^*\cdots\W_1^*\X=\textsf{svd}_{d_{\min}}(\tY-\bLambda^*)$ by Lemma \ref{lemma: local global}.
Therefore, the dual conditions in \eqref{equ: dual conditions for general problem} are implied by
\begin{flalign*}
\label{equ: stronger dual conditions}
\begin{split}
&\mbox{(1)}\quad \bLambda^*\in\partial K(\W_H^*\W_{H-1}^*\cdots\W_1^*\X)\cap\mathsf{Row}(\X);\\
&\mbox{(2)}\quad \cP_\cT(\tY-\bLambda^*)=\W_H^*\W_{H-1}^*\cdots\W_1^*\X;\\
&\mbox{(3)}\quad \|\cP_{\cT^\perp}(\tY-\bLambda^*)\|\le \sigma_{d_{\min}}(\W_H^*\W_{H-1}^*\cdots\W_1^*\X).
\end{split}
\end{flalign*}

It thus suffices to construct a dual certificate $\bLambda^*$ such that conditions (1), (2) and (3) hold, because conditions (1), (2) and (3) are stronger than conditions (a), (b) and (c). Let $r=\rank(\tY)$ and $\bar r=\min\{r,d_{\min}\}$. To proceed, we need the following lemma.

\medskip
\begin{lemma}[\cite{udell2016generalized}]
\label{lemma: closed-form solution}
Suppose $\tY\in\mathsf{Row}(\X)$. Let $(\W_1^*,\ldots,\W_H^*)$ be the solution to problem \eqref{equ: nuclear norm form} and let $\U\diag(\sigma_1(\tY),\ldots,\sigma_{r}(\tY))\V^T$ denote the skinny SVD of $\tY\in\mathsf{Row}(\X)$. We have $\W_H^*\W_{H-1}^*\cdots\W_1^*\X=\U\diag((\sigma_1(\tY)-\gamma)_+,\ldots,(\sigma_{\bar r}(\tY)-\gamma)_+,0,\ldots,0)\V^T$.
\end{lemma}

Recall that the sub-differential of the nuclear norm of a matrix $\Z$ is
\begin{equation*}
\partial_{\Z}\|\Z\|_*=\{\U_\Z\V_\Z^T+\T_\Z:\T_\Z\in\cT^\perp,\|\T_\Z\|\le 1\},
\end{equation*}
where $\U_\Z\bSigma_\Z\V_\Z^T$ is the skinny SVD of the matrix $\Z$.
So with Lemma \ref{lemma: closed-form solution}, the sub-differential of (scaled) nuclear norm at optimizer $\W_H^*\W_{H-1}^*\cdots\W_1^*\X$ is given by
\begin{equation}
\label{equ: subgradient}
\partial (\gamma\|\W_H^*\W_{H-1}^*\cdots\W_1^*\X\|_*)=\{\gamma\U_{:,1:{\bar r}}\V_{:,1:{\bar r}}^T+\T:\T\in\cT^\perp,\|\T\|\le \gamma\}.
\end{equation}
To construct the dual certificate, we set
\begin{equation*}
\bLambda^*=\underbrace{\gamma\U_{:,1:{\bar r}}\V_{:,1:{\bar r}}^T}_{\text{Component in space }\cT}+\underbrace{\U_{:,({\bar r}+1):r}\diag(\gamma,\ldots,\gamma)\V_{:,({\bar r}+1):r}^T}_{\text{Component $\T$ in space } \cT^\perp\text{ with }\|\T\|\le\gamma}\in\mathsf{Row}(\X),
\end{equation*}
where $\bLambda^*\in\mathsf{Row}(\X)$ because $\V^T\in\mathsf{Row}(\X)$ (This is because $\V^T$ is the right singular matrix of $\widetilde\Y$ and $\tY\in\mathsf{Row}(\X)$). So condition (1) is satisfied according to \eqref{equ: subgradient}. To see condition (2), $\cP_\cT(\tY-\bLambda^*)=\cP_\cT\tY-\gamma\U_{:,1:{\bar r}}\V_{:,1:{\bar r}}^T=\U\diag((\sigma_1(\tY)-\gamma)_+,\ldots,(\sigma_{{\bar r}}(\tY)-\gamma)_+,0,0,\ldots,0)\V^T=\W_H^*\W_{H-1}^*\ldots\W_1^*\X$, where the last equality is by Lemma \ref{lemma: closed-form solution} and the assumption $\sigma_{\min}(\tY)>\gamma$. As for condition (3), note that
\begin{equation*}
\begin{split}
\left\|\cP_{\cT^\perp}(\tY-\bLambda^*)\right\|&=\left\|\U_{:,({\bar r}+1):r}\diag(\sigma_{{\bar r}+1}(\tY)-\gamma,\ldots,\sigma_{r}(\tY)-\gamma)\V_{:,({\bar r}+1):r}^T\right\|\\
&=
\begin{cases}
0, & \text{if $\bar r=r$},\\
\sigma_{d_{\min}+1}(\tY)-\gamma, & \text{otherwise}.
\end{cases}
\end{split}
\end{equation*}
By Lemma \ref{lemma: closed-form solution}, $\sigma_{d_{\min}}(\W_H^*\W_{H-1}^*\cdots\W_1^*\X)\ge \|\cP_{\cT^\perp}(\tY-\bLambda^*)\|$. So the proof of strong duality is completed, where the dual problem is given in Section \ref{section: Dual and Bi-Dual Problems}.

To see the relation between the solutions of primal and dual problems, it is a direct result of Lemmas \ref{lemma: variational form of nuclear norm} and \ref{lemma: f and k local minimum}.

\section{Dual Problem of Deep Linear Neural Network}
\label{section: Dual and Bi-Dual Problems}
In this section, we derive the dual problem of non-convex program \eqref{equ: regularized PCA}. Denote by $G(\W_1,\ldots,\W_H)$ the objective function of problem \eqref{equ: regularized PCA}. Let $K(\cdot)=\gamma\|\cdot\|_*$, and let $\tY=\Y\X^\dag\X$ be the projection of $\Y$ on the row span of $\X$. We note that
\begin{equation*}
\begin{split}
&\quad \min_{\W_1,\ldots,\W_H}G(\W_1,\ldots,\W_H) - \frac{1}{2}\|\Y-\tY\|_F^2\\
&=\min_{\W_1,\ldots,\W_H}\frac{1}{2}\|\Y-\W_H\cdots\W_1\X\|_F^2 - \frac{1}{2}\|\Y-\tY\|_F^2+K(\W_H\cdots\W_1\X)\\
&= \min_{\W_1,\ldots,\W_H}\frac{1}{2}\|\tY-\W_H\cdots\W_1\X\|_F^2+K^{**}(\W_H\cdots\W_1\X)\\
&= \min_{\W_1,\ldots,\W_H}\max_{\mathsf{Row}(\bLambda)\subseteq\mathsf{Row}(\X)}\frac{1}{2}\|\tY-\W_H\cdots\W_1\X\|_F^2+\langle\bLambda,\W_H\cdots\W_1\X\rangle-K^*(\bLambda)\\
&= \min_{\W_1,\cdots,\W_H}\max_{\mathsf{Row}(\bLambda)\subseteq\mathsf{Row}(\X)} \frac{1}{2}\|\tY-\bLambda-\W_H\cdots\W_1\X\|_F^2-\frac{1}{2}\|\bLambda\|_F^2-K^*(\bLambda)+\langle\tY,\bLambda\rangle,
\end{split}
\end{equation*}
where the second equality holds since $K(\cdot)$ is closed and convex w.r.t. the argument $\W_H\W_{H-1}\cdots\W_1\X$ and the third equality is by the definition of conjugate function of nuclear norm.
Therefore, the dual problem is given by
\begin{equation*}
\begin{split}
&\max_{\mathsf{Row}(\bLambda)\subseteq\mathsf{Row}(\X)} \min_{\W_1,\ldots,\W_H} \frac{1}{2}\|\tY\hspace{-0.1cm}-\hspace{-0.1cm}\bLambda\hspace{-0.1cm}-\hspace{-0.1cm}\W_H...\W_1\X\|_F^2\hspace{-0.05cm}-\hspace{-0.05cm}\frac{1}{2}\|\bLambda\|_F^2\hspace{-0.05cm}-\hspace{-0.05cm}K^*(\bLambda)+\langle\tY,\bLambda\rangle+ \frac{1}{2}\|\Y-\tY\|_F^2\\
&=\max_{\mathsf{Row}(\bLambda)\subseteq\mathsf{Row}(\X)} \frac{1}{2}\sum_{i=d_{\min}+1}^{\min\{d_H,n\}}\sigma_i^2(\tY-\bLambda)-\frac{1}{2}\|\tY-\bLambda\|_F^2-K^*(\bLambda)+\frac{1}{2}\|\Y\|_F^2\\
&=\max_{\mathsf{Row}(\bLambda)\subseteq\mathsf{Row}(\X)} -\frac{1}{2}\|\tY-\bLambda\|_{d_{\min}}^2-K^*(\bLambda)+\frac{1}{2}\|\Y\|_F^2,
\end{split}
\end{equation*}
where $\|\cdot\|_{d_{\min}}^2=\sum_{i=1}^{d_{\min}}\sigma_i^2(\cdot)$. We note that
\begin{equation*}
K^*(\bLambda)=
\begin{cases}
0, & \|\bLambda\|\le \gamma;\\
+\infty, & \|\bLambda\|> \gamma.
\end{cases}
\end{equation*}
So the dual problem is given by
\begin{equation}
\label{equ: app dual problem}
\max_{\mathsf{Row}(\bLambda)\subseteq\mathsf{Row}(\X)} -\frac{1}{2}\|\tY-\bLambda\|_{d_{\min}}^2+\frac{1}{2}\|\Y\|_F^2,\quad\textup{s.t.}\quad \|\bLambda\|\le\gamma.
\end{equation}

Problem \eqref{equ: app dual problem} can be solved efficiently due to their convexity. In particular, Grussler et al.~\cite{grussler2016low} provided a computationally efficient algorithm to compute the proximal operators of functions $\displaystyle\frac{1}{2}\|\cdot\|_r^2$. Hence, the Douglas-Rachford algorithm can find the global minimum up to an $\epsilon$ error in function value in time $\textsf{poly}(1/\epsilon)$~\cite{he20121}.

\comment{
\newpage
\section{Quasi-Convexity}

Given $\epsilon>0$, we will show that
\begin{equation}
\widehat{\textsf{err}}_m(\w_1,...,\w_n)=\frac{1}{m}\sum_{i=1}^m \left(y_i-\sum_{k=1}^n\phi(\langle\w_k,\x_i)\rangle\right)^2
\end{equation}
is $(\epsilon,e^W,\{\w_1^*,...,\w_n^*\})$-SLQC at every $\W\in\mathbb{B}(\0,W)$, where $\W:=\{\w_1,...,\w_n\}$. Recall that $\phi(z)=\frac{1}{1+e^{-z}}$ and $\phi'(z)=\frac{e^{-z}}{(1+e^{-z})^2}$, and consider $\|\W\|\le W$ such that
\begin{equation}
\widehat{\textsf{err}}_m(\W)=\frac{1}{m}\sum_{i=1}^m \left(y_i-\sum_{k=1}^n\phi(\langle\w_k,\x_i)\rangle\right)^2\ge \epsilon.
\end{equation}
Denote by $\V$ a point which is $\epsilon/e^W$ close to the minimizer $\W^*$. We therefore have
\begin{equation}
\begin{split}
&\langle\nabla \widehat{\textsf{err}}_m(\W),\W-\V\rangle=\sum_{k=1}^n\langle\partial_{\w_k} \widehat{\textsf{err}}_m(\W),\w_k-\v_k\rangle\\
&=\sum_{k=1}^n\frac{2}{m}\sum_{i=1}^m\left(\sum_{k=1}^n\phi(\langle\w_k,\x_i)\rangle-y_i\right)\phi'(\langle\w_k,\x_i\rangle)(\langle\w_k,\x_i\rangle-\langle\v_k,\x_i\rangle)\\
&=\sum_{k=1}^n\frac{2}{m}\sum_{i=1}^m\frac{e^{-\langle\w_k,\x_i\rangle}}{(1+e^{-\langle\w_k,\x_i\rangle})^2}\left(\sum_{k=1}^n\phi(\langle\w_k,\x_i\rangle)-y_i\right)(\langle\w_k,\x_i\rangle-\langle\v_k,\x_i\rangle)\\
&=\sum_{k=1}^n\frac{2}{m}\sum_{i=1}^m\frac{e^{-\langle\w_k,\x_i\rangle}}{(1+e^{-\langle\w_k,\x_i\rangle})^2}\left(\sum_{k=1}^n\phi(\langle\w_k,\x_i\rangle)-\sum_{k=1}^n\phi(\langle\w_k^*,\x_i\rangle)\right)(\langle\w_k,\x_i\rangle-\langle\w_k^*,\x_i\rangle+\langle\w_k^*-\v_k,\x_i\rangle)\\
&=\sum_{k=1}^n\frac{2}{m}\sum_{i=1}^m\frac{e^{-\langle\w_k,\x_i\rangle}}{(1+e^{-\langle\w_k,\x_i\rangle})^2}\left(\sum_{k=1}^n\phi(\langle\w_k,\x_i\rangle)-\sum_{k=1}^n\phi(\langle\w_k^*,\x_i\rangle)\right)(\langle\w_k,\x_i\rangle-\langle\w_k^*,\x_i\rangle)\\
&\quad+\sum_{k=1}^n\frac{2}{m}\sum_{i=1}^m\frac{e^{-\langle\w_k,\x_i\rangle}}{(1+e^{-\langle\w_k,\x_i\rangle})^2}\left(\sum_{k=1}^n\phi(\langle\w_k,\x_i\rangle)-\sum_{k=1}^n\phi(\langle\w_k^*,\x_i\rangle)\right)\langle\w_k^*-\v_k,\x_i\rangle\\
\end{split}
\end{equation}
}

\end{document}